%% file: main.tex
\documentclass[10pt]{article} 
\usepackage[preprint]{tmlr}

\input{math_commands.tex}

\usepackage{bm}
\usepackage{hyperref}
\usepackage{url}   
\usepackage{amsmath, amssymb, amsthm, amsfonts}
\usepackage{array}
\usepackage{algorithm}
\usepackage[noend]{algpseudocode}
\usepackage{natbib}    
\usepackage{caption}   
\usepackage{multirow}
\usepackage{tabularx}
\usepackage{makecell}
\usepackage{tcolorbox}
\tcbuselibrary{listings}
\usepackage{float}
\usepackage{ragged2e}
\usepackage{graphicx}
\usepackage{cuted}
\usepackage{xcolor}    
\usepackage{enumitem}
\usepackage{subcaption}
\usepackage{booktabs}
\usepackage{url}
\usepackage{etoolbox}

\newtheorem{theorem}{Theorem}[section]
\newtheorem{lemma}[theorem]{Lemma}
\newtheorem{corollary}[theorem]{Corollary}
\newtheorem{proposition}[theorem]{Proposition}

\AtBeginEnvironment{appendices}{%
}

\title{AQUA: \underline{A}ttention via \underline{QU}ery m\underline{A}gnitudes for Memory and Compute Efficient Inference in LLMs}

\author{\name Santhosh G S \email santhoshgs013@gmail.com \\
      \addr Centre for Responsible AI\\
      Indian Institute of Technology Madras
      \AND
      \name Saurav Prakash \email saurav@ee.iitm.ac.in \\
      \addr Department of Electrical Engineering\\
      Indian Institute of Technology Madras
      \AND
      \name Balaraman Ravindran \email ravi@dsai.iitm.ac.in\\
      \addr Wadhwani School of Data Science and Artificial Intelligence\\
      Indian Institute of Technology Madras}

\def\month{MM}  
\def\year{YYYY} 
\def\openreview{\url{https://openreview.net/forum?id=XXXX}} 

\begin{document}

\maketitle

\begin{abstract}
The quadratic complexity of the attention mechanism remains a fundamental barrier to scaling Large Language Models (LLMs) to longer contexts, creating a critical bottleneck in both computation and memory. To address this, we introduce \textbf{AQUA} (\textbf{A}ttention via \textbf{QU}ery m\textbf{A}gnitudes) a novel and versatile approximation strategy that significantly reduces the cost of attention with a graceful performance trade-off. Our method operates in two phases: an efficient offline step where we compute a universal, language agnostic projection matrix via SVD on a calibration dataset, and an online inference step where we project query and key vectors and dynamically select a sparse subset of dimensions based on the query's magnitude. We provide a formal theoretical analysis of AQUA, establishing the break-even point at which it becomes more computationally efficient than standard attention. Our empirical evaluations on state-of-the-art models like Llama-3.1-8B demonstrate that a 25\% reduction in the attention dot-product computation can be achieved with a statistically insignificant impact on performance across a wide range of benchmarks. We further showcase the versatility of AQUA by demonstrating its ability to synergistically accelerate existing token eviction methods like H2O and to directly reduce KV-cache memory size. By offering a controllable knob to balance efficiency and accuracy, AQUA provides a practical and powerful tool for making large-scale LLM inference more accessible and sustainable.
\end{abstract}

\section{Introduction}

Large Language Models (LLMs) have rapidly become a transformative force in Artificial Intelligence, fueling the pursuit of Agentic AI - autonomous systems capable of tackling complex tasks with minimal human guidance \citep{sapkota2025aiagentsvsagentic}. However, realizing this ambitious vision hinges on the ability to process vast contexts, often spanning millions of tokens, which in turn creates an immense demand for computational and memory resources \citep{bommasani2022opportunitiesrisksfoundationmodels}. At the heart of this challenge is the Transformer's attention mechanism \citep{vaswani2023attentionneed}. Even with their success, attention's computational cost scales quadratically with sequence length. This scaling issue has become a fundamental bottleneck, posing a significant barrier to the continued advancement and deployment of ever-larger models \citep{tay2022efficienttransformerssurvey, beltagy2020longformerlongdocumenttransformer}.

A considerable amount of research has been done in mitigating these challenges, with most efforts focusing on reducing either the memory footprint or the computational load \citep{tay2022efficienttransformerssurvey}. Few works, however, have worked to address both simultaneously. Existing strategies often target specific components of the Transformer architecture, such as pruning MLP layers or approximating the attention mechanism itself \citep{zhong2025blockprunerfinegrainedpruninglarge}. Within the attention layer, these approximations may focus solely on the attention weights or extend to the final attention outputs \citep{choromanski2022rethinkingattentionperformers}.

The present auto-regressive inferencing mechanism trades-off between memory and computation, using a method called Key-Value (KV) caching. To avoid re-computing the key and value vectors for all previous tokens at each new decoding step, models cache these activations \citep{hichri_2025, chen2024naclgeneraleffectivekv}. This practice is essential for the real-time, responsive performance of modern LLMs and drastically reduces computational load. However, the memory required to store this cache grows linearly with the sequence length. For very large contexts, the memory footprint of the KV-cache can even supersede the memory required to load the model's weights for inference, creating a severe memory bottleneck \citep{yan2025lazymaracceleratingmaskedautoregressive}. This complex interplay highlights a pressing need for methods that can optimize the attention mechanism at inference time, allowing practitioners to flexibly budget between compute, memory, and accuracy based on their specific use-case and application.

\begin{figure*}[t]
\centering
\includegraphics[width=1.0\textwidth]{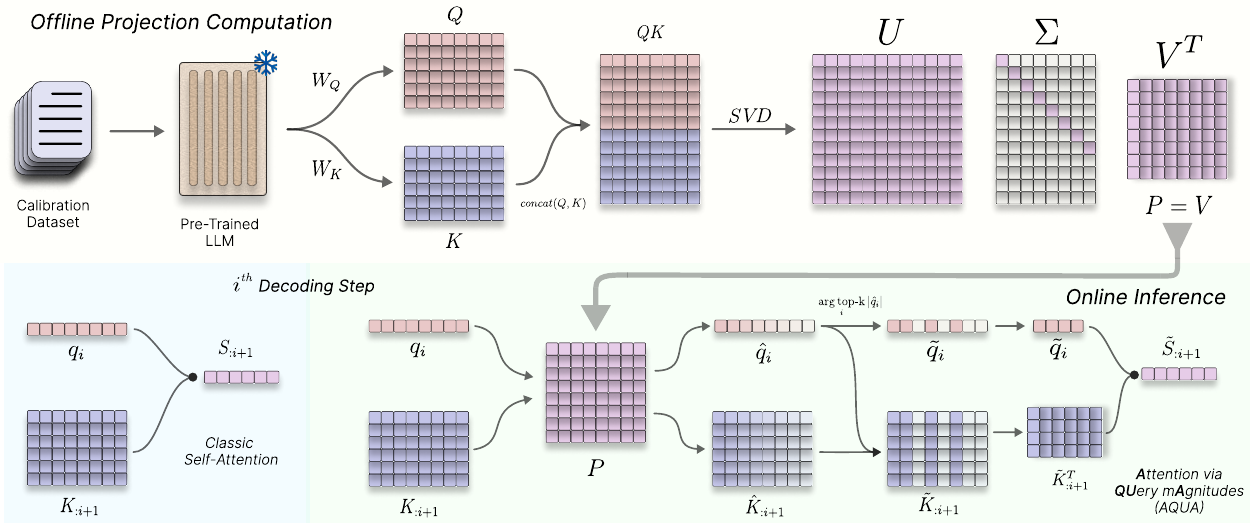}
\caption{A schematic of AQUA, illustrating the two-phase process: (Top) Offline computation of a universal projection matrix P, and (Bottom) Online inference using projected vectors and magnitude-based dimension selection.}
\label{fig:qma_overview}
\end{figure*}

In this work, we introduce \textbf{A}ttention via \textbf{QU}ery m\textbf{A}gnitudes for Memory and
Compute Efficient Inference in LLMs (\textbf{AQUA}), a novel approach designed to precisely fill this gap, concurrently reducing both the computational and memory demands of the attention mechanism. As illustrated in Figure \ref{fig:qma_overview}, our method targets the attention weights, motivated by our finding that the query and key vectors can be efficiently transformed into a sparser representation. In this new space, we can perform an informed pruning of dimensions with the lowest magnitudes using transformed queries as references. Our empirical analysis reveals that even after pruning 25\% of the lowest-magnitude dimensions from the query and key vectors, the resulting loss in accuracy on various benchmarks is, on average is less than 1\%. 

Similar to \citet{singhania2024lokilowrankkeysefficient}, we use an universal projection matrix, which we compute offline. However, instead of being exclusively based on key vectors, our projection matrix utilizes information from both key vectors and query vectors combined together. This allows us to project the query and key vectors to an aligned low dimensional space, so when we prune using magnitudes of queries, the dimensions get aligned on keys as well, which leads to performance retention even on aggressive pruning. Exploiting this low dimensional property of the rotated queries and keys to prune lesser magnitude components, we compute attention weights only on the remaining heavy hitter components. We detail the methodology for its computation in Section \ref{sec:proj_matrix}. Furthermore, we theoretically prove that for sequences exceeding a certain length, our method yields progressively increasing computational savings. We also demonstrate that our method is highly versatile: it can be used as a standalone attention approximation strategy for direct inference, or it can be integrated as a complementary component on top of existing token eviction strategies to further reduce the memory requirements by eliminating sparsely weighted tokens corresponding to approximated attention weights.

So, putting it all together, in this work we make the following contributions:
\begin{itemize}[leftmargin=*, noitemsep]
\item We propose a novel attention approximation strategy, named AQUA, which reduces the dimensionality of query and key vectors based on magnitude. This method can function as a standalone replacement for standard attention or be integrated with other token eviction strategies to improve their efficiency.
\item We present an empirical justification for using a globally calibrated, offline projection matrix, demonstrating that this approach is significantly more efficient than computing an exact, online projection via SVD without a substantial loss in performance.
\item We provide a formal theoretical analysis that establishes the computational break-even point where our method begins to outperform standard attention, proving that the performance gains increase as more tokens are decoded.
\item We conduct a comprehensive benchmark evaluation, comparing our standalone method against standard attention and demonstrating its synergistic performance improvements when applied on top of existing token eviction strategies, evaluated on both perplexity and downstream task-based metrics.
\end{itemize}

\section{Related Work}

The challenge of optimizing Large Language Model (LLM) inference has spurred a wide array of research. While a broad survey of related paradigms is available in Appendix~\ref{app:related_work_survey}, this section focuses on the closest set of works in attention approximation and token eviction that inform our approach.

\subsection*{Attention Approximation Techniques}
This line of research seeks to reduce computational cost by approximating the attention mechanism. A notable recent approach is EigenAttention \citep{saxena2024eigenattentionattentionlowrank}, which compresses the KV-cache by decomposing the projection weights for K and V into low-rank factors. While effective at reducing memory, its compression rank is a fixed hyperparameter that must be decided offline. Our work is similar in its goal of reducing dimensionality but differs by operating on projected vectors at runtime with a dynamic selection mechanism.

Two other highly relevant works are SparQ Attention and LoKi Attention. SparQ Attention \citep{ribar2024sparqattentionbandwidthefficientllm} also uses query magnitudes for approximation but requires costly non-contiguous memory access and increases the memory footprint by 50\%. LoKi Attention \citep{singhania2024lokilowrankkeysefficient} uses an offline projection matrix similar to ours but relies on a static slicing of the trailing dimensions, a strategy we empirically show to be suboptimal. Our work, AQUA, builds on these insights, combining the efficiency of an offline projection with a more effective dynamic magnitude selection, all while avoiding the overheads of prior methods.

\subsection*{Token Eviction Techniques}
A popular strategy for managing long contexts is to prune the KV-cache by evicting tokens. A seminal work in this area is H2O \citep{zhang2023h2oheavyhitteroracleefficient}, which identifies ``Heavy Hitter'' (H2) tokens by monitoring their accumulated attention scores. Its core innovation is a KV-cache eviction policy that dynamically retains a balance of these important H2 tokens alongside the most recent tokens, recognizing that both are crucial for maintaining context. While highly effective, this approach permanently discards non-H2 tokens, risking information loss. Our method is complementary; as shown in our experiments, we can accelerate the identification of these Heavy Hitters by first computing an approximate attention score, thereby enhancing the efficiency of the eviction policy itself.

\begin{figure*}[t]
\centering
\includegraphics[width=1.0\textwidth]{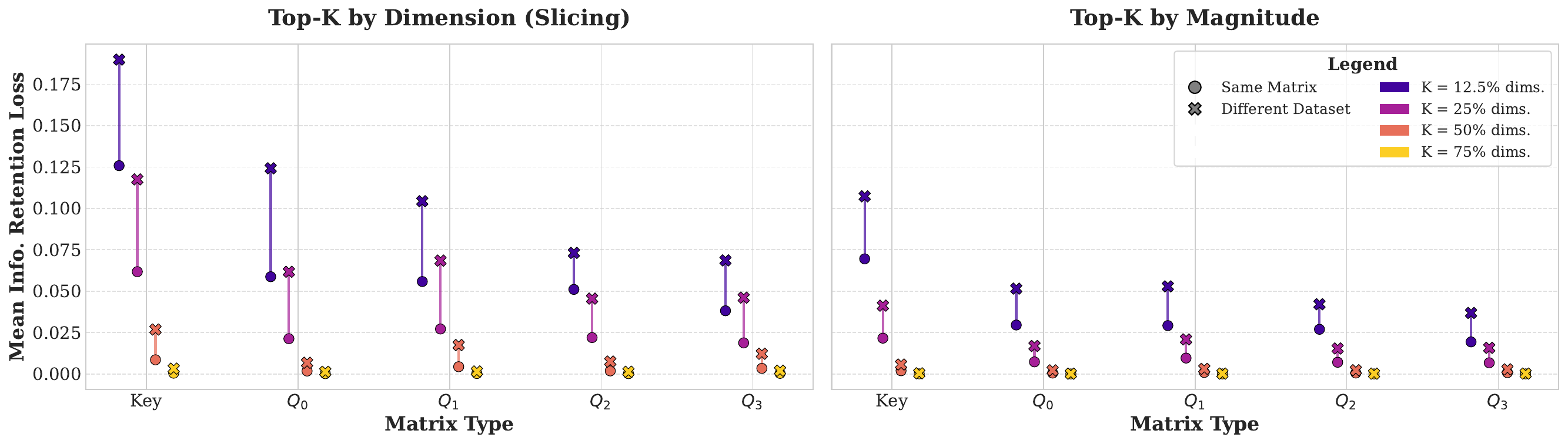}
\caption{Comparison of mean information retention loss for two projection matrix sources (Online ``Same Matrix'' SVD vs. Offline ``Different Dataset'' SVD) and two dimension selection methods (``Top-K by Dimension'' vs. ``Top-K by Magnitude''). The analysis is on Layer 0, Head 0 of \texttt{Llama-3.1-8B} \citep{grattafiori2024llama3herdmodels}. The minimal gap between the ``Same Matrix'' and ``Different Dataset'' points validates our offline calibration approach. The significant reduction in loss for ``Top-K by Magnitude'' provides strong evidence for its superiority over naive slicing.}
\label{fig:dumbbell_comparison}
\vspace{-1em}
\end{figure*}

\section{Primer on Classic Attention \& Notations}

The self-attention mechanism is a cornerstone of the Transformer architecture \citep{vaswani2023attentionneed}. In the context of auto-regressive decoding, where tokens are generated sequentially, the attention mechanism computes a context vector by attending to all previously generated tokens in the sequence \citep{vaswani2023attentionneed}. This section formalizes the step-by-step process of classic (or standard) attention for a single attention head and establishes the notations used throughout this paper.

\subsection{Core Components and Notations}

Let us define the primary dimensions and matrices involved in the attention computation:
\begin{itemize}
    \item $d_{model}$: The primary embedding dimension of the model.
    \item $d_{head}$: The dimension of a single attention head's query, key, and value vectors.
    \item $W_Q, W_K, W_V \in \mathbb{R}^{d_{model} \times d_{head}}$: The learnable weight matrices used to project the input embeddings into the query, key, and value spaces, respectively \citep{vaswani2023attentionneed}.
\end{itemize}

For the generation of the $(i+1)^{th}$ token in a sequence (where $i$ is the current time step, starting from $i=0$), we consider the input embedding $x_i \in \mathbb{R}^{1 \times d_{model}}$ corresponding to the token at position $i$ \citep{vaswani2023attentionneed}. From this, we derive the following vectors:
\begin{itemize}
    \item $q_i \in \mathbb{R}^{1 \times d_{head}}$: The query vector for the current token, computed as $q_i = x_i W_Q$.
    \item $k_i \in \mathbb{R}^{1 \times d_{head}}$: The key vector for the current token, computed as $k_i = x_i W_K$.
    \item $v_i \in \mathbb{R}^{1 \times d_{head}}$: The value vector for the current token, computed as $v_i = x_i W_V$.
\end{itemize}

A critical component of efficient auto-regressive decoding is the Key-Value (KV) cache, which stores the key and value vectors from all previous time steps \citep{hichri_2025}. We denote the cached matrices up to, but not including, the current step $i$ as:
\begin{itemize}
    \item $K_{:i} \in \mathbb{R}^{i \times d_{head}}$: The matrix of key vectors from tokens $0, \dots, i-1$.
    \item $V_{:i} \in \mathbb{R}^{i \times d_{head}}$: The matrix of value vectors from tokens $0, \dots, i-1$.
\end{itemize}

\subsection{The Auto-Regressive Attention Step}

At decoding step $i$, the model computes the attention output for the current token by performing the following sequence of operations:

\begin{enumerate}
    \item \textbf{Update the KV Cache:} The key and value vectors for the current token, $k_i$ and $v_i$, are concatenated to their respective cache matrices \citep{hichri_2025}.
    \begin{align}
        K_{:i+1} &= \text{concat}(K_{:i}, k_i) \in \mathbb{R}^{(i+1) \times d_{head}} \\
        V_{:i+1} &= \text{concat}(V_{:i}, v_i) \in \mathbb{R}^{(i+1) \times d_{head}}
    \end{align}

    \item \textbf{Compute Attention Scores:} The query vector $q_i$ is used to score each key in the updated key cache via a dot product. This step identifies which of the previous tokens are most relevant to the current one \citep{vaswani2023attentionneed}.
    \begin{equation}
        S = q_i K_{:i+1}^T \in \mathbb{R}^{1 \times (i+1)}
    \end{equation}
    The computational complexity of this operation is $O((i+1) \cdot d_{head})$, which grows linearly with the sequence length per token, but the overall self-attention layer computation is quadratic in the full sequence length, representing a significant bottleneck in LLMs \citep{keles2022computationalcomplexityselfattention}.

    \item \textbf{Scale and Normalize:} The scores are scaled by the inverse square root of the head dimension to prevent the dot products from growing too large, which could saturate the softmax function. A softmax is then applied to obtain a probability distribution, representing the attention weights \citep{nakanishi2025scalablesoftmaxsuperiorattention}.
    \begin{equation}
        A = \text{softmax}\left(\frac{S}{\sqrt{d_{head}}}\right) \in \mathbb{R}^{1 \times (i+1)}
    \end{equation}

    \item \textbf{Compute Context Vector:} The attention weights $A$ are used to compute a weighted sum of the value vectors in the value cache. This produces the context vector $c_i$, which summarizes the information from the preceding tokens relevant to the current step.
    \begin{equation}
        c_i = A V_{:i+1} \in \mathbb{R}^{1 \times d_{head}}
    \end{equation}
\end{enumerate}
This resulting context vector $c_i$ is then passed to subsequent layers of the Transformer decoder, forming the basis for predicting the next token in the sequence.

\section{AQUA Description}

The AQUA mechanism, illustrated in Figure \ref{fig:qma_overview}, is a two-phase process designed to optimize the standard attention computation by leveraging a pre-computed projection and dynamic, magnitude-based dimension selection.

The first phase, \textbf{Offline Projection Computation}, is performed once per model to generate a universal projection matrix $P$. The detailed methodology for constructing this matrix by collecting activations and performing SVD is elaborated in Section \ref{subsubsec:methodology_for_calibrating_P}.

The second phase, \textbf{Online Inference}, occurs at each decoding step. The bottom panel of Figure \ref{fig:qma_overview} contrasts our method with classic self-attention. Instead of computing attention on the full vectors, AQUA first projects the incoming query ($q_i$) and keys ($K_{:i+1}$) into the new space using the pre-computed matrix $P$. The core of our method is the next step: we dynamically identify the top-$k$ dimensions based on the absolute magnitude of the components in the projected query vector ($\hat{q}_i$). Finally, the approximate attention scores ($\tilde{S}_{:i+1}$) are computed using only this sparse subset of dimensions from both the query and keys, significantly reducing the computational cost of the dot product.

\subsection{Algorithm}

The online inference phase is formalized in Algorithm \ref{alg:query_mag_detailed}. At each decoding step, the algorithm takes the current query and key vectors, along with the existing key cache, as input. It begins by projecting the current vectors and updating the cache in the new coordinate space. The key step involves identifying the indices of the top-$k$ dimensions based on the magnitude of the projected query. The final, approximate attention scores are then computed using only these dynamically selected dimensions from both the query and the full key cache.

\begin{algorithm}[!t]
\caption{AQUA (for token $i{+}1$)}
\label{alg:query_mag_detailed}
\textbf{Input:} Current query $q_i \in \mathbb{R}^{1 \times d_{\text{head}}}$, key $k_i \in \mathbb{R}^{1 \times d_{\text{head}}}$, key cache $K_{:i} \in \mathbb{R}^{i \times d_{\text{head}}}$ \\
\textbf{Parameter:} Top-$k$ dims $k \in \{1,\dots,d_{\text{head}}\}$, projection matrix $P \in \mathbb{R}^{d_{\text{head}} \times d_{\text{head}}}$ \\
\textbf{Output:} Approximate attention scores $\tilde{S} \in \mathbb{R}^{1 \times (i+1)}$
\begin{algorithmic}[1]
\State $\hat{q}_i \gets q_i P$ \Comment{Project query}
\State $\hat{k}_i \gets k_i P$ \Comment{Project key}
\State $\hat{K}_{:i+1} \gets \text{concat}(\hat{K}_{:i}, \hat{k}_i)$ \Comment{Update projected key cache}
\State $v_{\text{mag}} \gets |\hat{q}_i|$ \Comment{Compute query magnitude}
\State $I_{\text{topk}} \gets \arg\text{TopK}(v_{\text{mag}})$ \Comment{Select top-$k$ dimensions}
\State $\tilde{q}_i \gets \hat{q}_i[:, I_{\text{topk}}]$ \Comment{Slice projected query}
\State $\tilde{K}_{:i+1} \gets \hat{K}_{:i+1}[:, I_{\text{topk}}]$ \Comment{Slice projected key cache}
\State $\tilde{S} \gets \tilde{q}_i \tilde{K}_{:i+1}^\top$ \Comment{Compute attention scores}
\State \Return $\tilde{S}$
\end{algorithmic}
\end{algorithm}

\section{Theoretical Performance Results}
\label{sec:analysis_results}

To formally ground our method, we analyze its computational complexity, focusing on the calculation of the unnormalized attention scores - the dot-product operation that is the primary target of our optimization and one of the main driver of cost in the attention mechanism. Our analysis establishes the conditions under which the AQUA method provides a clear performance advantage over the standard approach. The detailed proofs and derivations for the following results are provided in Appendix~\ref{app:detailed_proofs}.

First, we establish the baseline cost. In the standard auto-regressive setting, the complexity of computing the dot product between the current query and all keys in the cache is linear with respect to the sequence length \citep{vaswani2023attentionneed, tay2022efficienttransformerssurvey}.
\begin{itemize}
    \item \textbf{Standard Attention Cost:} $C_{std} = O((i+1) \cdot d_{head})$
\end{itemize}

Next, we formalize the complexity of our AQUA algorithm. This cost is composed of a fixed, one-time overhead for the vector projections ($O(d_{head}^2)$) and a variable cost for the final dot product, which scales with our reduced dimension, $k$.
\begin{itemize}
    \item \textbf{AQUA Cost:} $C_{AQUA} = O(d_{head}^2 + (i+1) \cdot k)$
\end{itemize}

By comparing these two complexities, we can derive the precise ``break-even point'' at which our method becomes more efficient. This critical result shows that for any meaningful approximation level ($k < d_{head}$), there exists a sequence length beyond which AQUA is guaranteed to be faster. The key condition is:
$$ i+1 > \frac{d_{head}^2}{d_{head} - k} $$
This inequality reveals the fundamental trade-off of our method: the fixed projection cost is amortized over the sequence, and the per-token savings of $(d_{head} - k)$ ensure that for sufficiently long contexts, our approach will always yield a net computational gain that grows with every new token.

\section{Computation and Validation of the Projection Matrix}
\label{sec:proj_matrix}

The central hypothesis of our method is that the salient information within query and key vectors can be concentrated into a smaller subset of dimensions. To achieve this, we must first find a suitable projection that transforms the original vector space into a new one where the dimensions are ordered by importance. This section details the methodology for computing such a projection matrix, $P$, and provides empirical evidence validating its effectiveness and generalizability.

\subsection{Methodology for Offline Calibration}

\subsubsection*{Motivation and the Ideal Online Approach}
Ideally, for any given set of key vectors $K_{:i+1}$ at a decoding step $i$, we would find a transformation that aligns the coordinate axes with the directions of maximum variance within that specific set of vectors. This is precisely the goal of Principal Component Analysis (PCA) \citep{MACKIEWICZ1993303}. The optimal projection matrix $P$ for this task can be found using Singular Value Decomposition (SVD) \citep{1102314} of the key cache $K_{:i+1}$.

However, this ``online'' approach is computationally infeasible. The complexity of computing SVD, $O(\min((i+1)d_{head}^2, (i+1)^2d_{head}))$  \citep{li2019tutorialcomplexityanalysissingular}, would introduce a prohibitive overhead at every decoding step, negating any potential gains. A detailed derivation of this complexity can be found in Appendix~\autoref{app:svd_proof}.

\subsubsection*{Proposed Offline Calibration Method}
\label{subsubsec:methodology_for_calibrating_P}
To circumvent this bottleneck, we propose computing a single, fixed projection matrix $P$ for each layer and head \textit{offline}. The procedure is as follows:
\begin{enumerate}
    \item \textbf{Curate a Calibration Dataset:} We select a large, representative corpus of text (e.g., \texttt{BookCorpus} \citep{Zhu_2015_ICCV}) and segment it into uniform long sequences (e.g., $N = 4096$ tokens).
    \item \textbf{Collect Activation Vectors:} We pass these sequences through the pre-trained model. For each layer and head, we collect a large number of query vectors ($q_i$) and key vectors ($k_i$) after all standard transformations have been applied.
    \item \textbf{Perform Global SVD:} For each layer and head, we concatenate the collected vectors to form a large data matrix, $D_{calib}$. We then perform SVD on this matrix: $D_{calib} = U\Sigma V^T$ \citep{1102314}.
    \item \textbf{Store the Projection Matrix:} The resulting matrix $V \in \mathbb{R}^{d_{head} \times d_{head}}$ contains the principal components that capture the most significant directions of variance across the entire calibration dataset. We save this matrix as the fixed projection matrix $P$ for that specific layer and head.
\end{enumerate}
By pre-computing $P$ offline, the expensive SVD operation is eliminated from the inference loop, leaving only the cost of two efficient vector-matrix multiplications.

\begin{figure}[t!]
\centering
\includegraphics[width=0.47\textwidth]{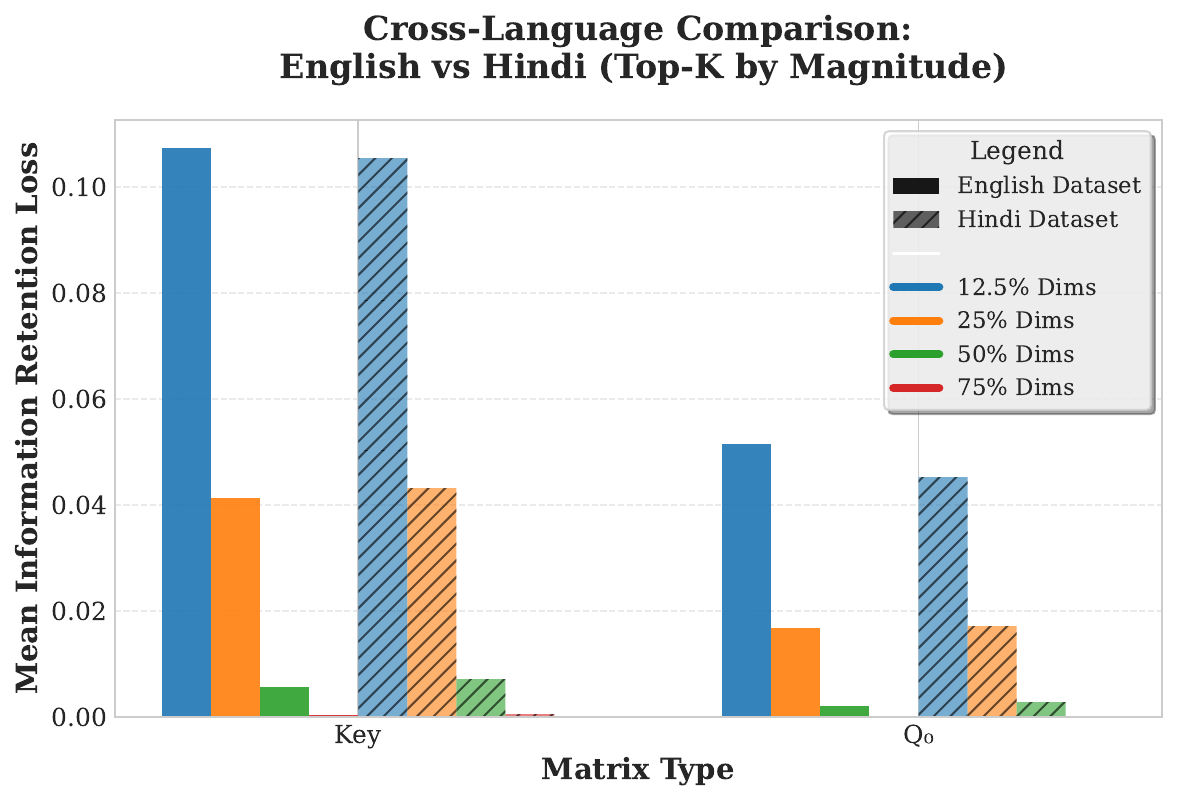} 
\caption{Cross-lingual comparison of mean information retention loss for the Key and first Query ($Q_0$) matrices. The similar loss profiles demonstrate the robustness of the English-calibrated projection matrix when applied to a different language and script.}
\label{fig:language_comparison_subset}
\vspace{-1em}
\end{figure}

\subsection{Empirical Validation of the Offline Approach}

To validate our offline calibration method, we must show that it does not introduce a significant loss of information compared to the ideal online approach. We quantify this using a metric we call as \textbf{Information Retention Loss}, $\mathcal{L}_{info}$ \citep{article, Greenacre_2022_PCA}. For an original vector $v \in \mathbb{R}^{d_{head}}$, its projected counterpart $\hat{v} = vP$, and a set of $k$ selected indices $I_k \subset \{1, \dots, d_{head}\}$, the loss is the normalized difference between the original vector's norm and the norm of its retained components: $ \mathcal{L}_{info}(v, \hat{v}, I_k) = \frac{\big| \|v\|_2 - \|\hat{v}[I_k]\|_2 \big|}{\|v\|_2} $

where $\hat{v}[I_k]$ denotes the vector containing only the components of $\hat{v}$ at the indices in $I_k$. A lower $\mathcal{L}_{info}$ indicates that the truncated, projected vector better preserves the ``energy'' of the original \citep{Greenacre_2022_PCA}.

Figure \ref{fig:dumbbell_comparison} compares the information retention loss under two conditions: using an online projection matrix computed from the \textbf{Same Matrix} versus our offline matrix calibrated on a \textbf{Different Dataset} (\texttt{BookCorpus} \citep{Zhu_2015_ICCV}, evaluated on \texttt{WikiText} \citep{merity2016pointer}). The results compellingly show that the loss incurred by using our pre-calibrated matrix is only marginally higher than that of the ideal, but impractical, online SVD. This validates that our offline approach is a highly effective and efficient proxy.

\begin{table*}[t!]
\centering
\caption{
    A summary of performance for \texttt{Llama-3.1-8B-Instruct} and \texttt{OLMoE-1B-7B-Instruct} under key levels of AQUA pruning. `B' denotes the baseline ($k_{ratio}=1.0$), with results shown in \textbf{bold}. The full table is in Appendix~\ref{app:full_standalone_results}.
}
\label{tab:pruning_results_percent_condensed}
\resizebox{\textwidth}{!}{%
\begin{tabular}{@{}llccccccc@{}}
\toprule
\multirow{2}{*}{Model} & \multirow{2}{*}{\shortstack{$k_{ratio}$}} & MMLU & GSM8K & HellaSwag & WinoGrande & \shortstack{TruthfulQA\\MC2} & \shortstack{ARC\\Challenge} & WikiText \\
\cmidrule(l){3-9} 
& & (acc $\uparrow$) & (acc $\uparrow$) & (acc $\uparrow$) & (acc $\uparrow$) & (acc $\uparrow$) & (acc $\uparrow$) & (ppl $\downarrow$) \\ 
\midrule
\multirow{4}{*}{\shortstack{Llama-3.1-8B \\ Instruct}} 
 & \textbf{B} & $\bm{0.687 \pm 0.004}$ & $\bm{0.816 \pm 0.011}$ & $\bm{0.785 \pm 0.004}$ & $\bm{0.755 \pm 0.012}$ & $\bm{0.551 \pm 0.016}$ & $\bm{0.647 \pm 0.014}$ & $\bm{8.910}$ \\
 & 0.75 & $0.685 \pm 0.004$ & $0.805 \pm 0.011$ & $0.785 \pm 0.004$ & $0.757 \pm 0.012$ & $0.551 \pm 0.016$ & $0.645 \pm 0.014$ & $8.930$ \\
 & 0.50 & $0.666 \pm 0.004$ & $0.720 \pm 0.012$ & $0.780 \pm 0.004$ & $0.738 \pm 0.012$ & $0.551 \pm 0.016$ & $0.620 \pm 0.014$ & $9.200$ \\
 & 0.30 & $0.507 \pm 0.004$ & $0.146 \pm 0.010$ & $0.732 \pm 0.004$ & $0.598 \pm 0.014$ & $0.502 \pm 0.016$ & $0.530 \pm 0.015$ & $12.550$ \\
\midrule
\multirow{4}{*}{\shortstack{OLMoE-1B-7B \\ Instruct}} 
 & \textbf{B} & $\bm{0.530 \pm 0.004}$ & $\bm{0.451 \pm 0.014}$ & $\bm{0.783 \pm 0.004}$ & $\bm{0.673 \pm 0.013}$ & $\bm{0.491 \pm 0.016}$ & $\bm{0.532 \pm 0.015}$ & $\bm{11.340}$ \\
 & 0.75 & $0.529 \pm 0.004$ & $0.453 \pm 0.014$ & $0.783 \pm 0.004$ & $0.669 \pm 0.013$ & $0.485 \pm 0.016$ & $0.542 \pm 0.015$ & $11.330$ \\
 & 0.50 & $0.526 \pm 0.004$ & $0.426 \pm 0.014$ & $0.778 \pm 0.004$ & $0.658 \pm 0.013$ & $0.488 \pm 0.016$ & $0.540 \pm 0.015$ & $11.340$ \\
 & 0.30 & $0.485 \pm 0.004$ & $0.252 \pm 0.012$ & $0.747 \pm 0.004$ & $0.615 \pm 0.014$ & $0.481 \pm 0.016$ & $0.513 \pm 0.015$ & $12.140$ \\
\bottomrule
\end{tabular}%
}
\end{table*}

\subsection{Generalizability and Extension to GQA}

\subsubsection*{Cross-Lingual Generalizability}
A key question is whether the learned projection matrix is language-agnostic or if it overfits to the linguistic properties of the calibration data. To test this, we applied our projection matrix, calibrated on English text (\texttt{BookCorpus} \citep{Zhu_2015_ICCV}), to query and key vectors generated from a dataset in a completely different script: Hindi (\texttt{wikipedia-hi} \citep{zicsx_wikipedia_hindi}).

As shown in Figure \ref{fig:language_comparison_subset}, the information retention loss profiles for English and Hindi are remarkably similar. The surprising lack of degradation strongly suggests that the principal components we capture are not tied to a specific language but rather reflect fundamental, language-agnostic properties of the attention heads themselves. A detailed discussion of the experimental design and the full results across all matrices for the GQA group are provided in Appendix ~\autoref{app:detailed_cross_linguality}.

\subsubsection*{Extension to Grouped-Query Attention (GQA)}
Our method naturally extends to modern architectures like \texttt{Llama-3.1-8B-Instruct} that use Grouped-Query Attention (GQA) \citep{grattafiori2024llama3herdmodels}, where a group of $N_Q$ query heads shares a single key head. To create a shared projection matrix for the group, we must capture the collective variance of all constituent heads. Let $D_{q_j} \in \mathbb{R}^{M \times d_{head}}$ be the matrix of $M$ query vectors collected for the $j$-th query head in the group, and let $D_k \in \mathbb{R}^{M \times d_{head}}$ be the matrix of corresponding vectors from the shared key head. The group's calibration matrix, $D_{calib}^{GQA}$, is formed by vertically stacking these individual matrices:
$$ D_{\text{calib}}^{\text{GQA}} = \begin{pmatrix} D_{q_1} & D_{q_2} & \dots & D_{q_{N_Q}} & D_k \end{pmatrix}^T \in \mathbb{R}^{((N_Q+1)M) \times d_{\text{head}}} $$
Performing SVD on this combined matrix ($D_{calib}^{GQA} = U\Sigma V^T$) yields a single projection matrix $P = V$ for the entire group. This approach not only reduces the memory required for storing projection matrices but also ensures the projection is informed by the shared statistical properties of the group \citep{chen2024naclgeneraleffectivekv}. The analysis in Figure \ref{fig:dumbbell_comparison} was conducted on such a GQA group (Layer 0, Head 0, with $N_Q=4$), where the four query matrices ($Q_0$ to $Q_3$) and the key matrix all use this shared projection.

\subsubsection{Rotational Invariance of Attention Scores}

A crucial property of using an orthogonal matrix for projection is that the projection itself is a lossless rotation. It does not alter the underlying dot product scores. This means that the approximation error in our method is introduced only by the subsequent truncation of dimensions (i.e., selecting the top-$k$), not by the initial projection. A lemma to formalize this property can be found in Appendix ~\autoref{app:rotational_invariance}.

\vspace{-0.5em}

\section{Justification for Magnitude-Based Dimension Selection}

The results in Figure \ref{fig:dumbbell_comparison} not only validate our offline approach but also reveal a second, more critical insight: the method used to select the $k$ dimensions after projection has a profound impact on performance. A naive approach would be to simply slice the first $k$ dimensions, assuming they inherently contain the most information. Our analysis demonstrates that a dynamic, magnitude-based selection is vastly superior.

\vspace{-0.5em}

\subsection{The Flaw in Naive Slicing: A Mismatch of Importance}
The core issue with naive slicing is that it conflates two different notions of ``importance'': global variance versus token-specific activity. PCA identifies dimensions that are important \textit{globally} by capturing the most variance across an entire dataset. However, for any individual query, the most important dimensions are those that are most ``active'' for that specific token, which is best measured by their magnitude \citep{ashkboos2024slicegptcompresslargelanguage}.

Our empirical analysis confirms this mismatch. We found that the overlap between the top dimensions by magnitude and the leading principal components is often surprisingly low. For instance, selecting the top 12.5\% of dimensions by magnitude does not guarantee they are captured even within the top 25\% of principal components. This directly shows that the most active dimensions for a given token are not necessarily the first few principal components. These results can be found in ~\autoref{fig:key_query_overlap} as a part of Appendix~\ref{app:magnitude_vs_pca}, where we detailed methodology and the full overlap analysis.

\subsection{Magnitude Selection Halves the Information Loss}
The practical consequence of this mismatch is evident in Figure \ref{fig:dumbbell_comparison}. When comparing the two selection methods, ``Top-K by Dimension (Slicing)'' versus ``Top-K by Magnitude'', the information retention loss is consistently reduced by approximately half when using our magnitude-based approach. By dynamically selecting the most active dimensions for each specific vector, we preserve the vector's energy far more effectively than a static slicing strategy could. This provides a clear and compelling justification for the central mechanism of our AQUA algorithm.

\begin{table*}[t!]
\centering
\caption{
    Performance of \texttt{Llama-3.1-8B-Instruct} \citep{grattafiori2024llama3herdmodels} using the synergistic \textbf{AQUA-H2O} attention mechanism.
    The table shows results while tuning the H2O Heavy Hitter Ratio ($H2O_{ratio}$) \citep{zhang2023h2oheavyhitteroracleefficient} and the AQUA Ratio ($k_{ratio}$). The baseline H2O performance ($H2O_{ratio}=1.00$) is denoted by `B'.
    Performance is measured by accuracy (acc, higher is better) and perplexity (ppl, lower is better).
    Values are reported as $mean \pm standard\_error$.
}
\label{tab:qma-h2o_results_percent}
\resizebox{\textwidth}{!}{%
\begin{tabular}{@{}ccccccccc@{}}
\toprule
\multicolumn{2}{c}{Hyperparameters} & \multicolumn{7}{c}{Benchmark Performance} \\
\cmidrule(r){1-2} \cmidrule(l){3-9}
$H2O_{ratio}$ & $k_{ratio}$ & MMLU & GSM8K & HellaSwag & WinoGrande & \shortstack{TruthfulQA\\MC2} & \shortstack{ARC\\Challenge} & WikiText \\
 & & (acc $\uparrow$) & (acc $\uparrow$) & (acc $\uparrow$) & (acc $\uparrow$) & (acc $\uparrow$) & (acc $\uparrow$) & (ppl $\downarrow$) \\
\midrule
\multirow{4}{*}{0.25} 
 & 0.30 & $0.512 \pm 0.004$ & $0.133 \pm 0.009$ & $0.736 \pm 0.004$ & $0.590 \pm 0.014$ & $0.530 \pm 0.016$ & $0.532 \pm 0.015$ & $12.780$ \\
 & 0.50 & $0.666 \pm 0.004$ & $0.708 \pm 0.013$ & $0.782 \pm 0.004$ & $0.743 \pm 0.012$ & $0.557 \pm 0.016$ & $0.617 \pm 0.014$ & $9.250$ \\
 & 0.75 & $0.684 \pm 0.004$ & $0.798 \pm 0.011$ & $0.787 \pm 0.004$ & $0.754 \pm 0.012$ & $0.556 \pm 0.016$ & $0.644 \pm 0.014$ & $8.950$ \\
 & 1.00 & $0.686 \pm 0.004$ & $0.795 \pm 0.011$ & $0.786 \pm 0.004$ & $0.762 \pm 0.012$ & $0.559 \pm 0.016$ & $0.651 \pm 0.014$ & $8.930$ \\
\midrule
\multirow{4}{*}{0.50} 
 & 0.30 & $0.510 \pm 0.004$ & $0.147 \pm 0.010$ & $0.732 \pm 0.004$ & $0.591 \pm 0.014$ & $0.504 \pm 0.016$ & $0.537 \pm 0.015$ & $12.560$ \\
 & 0.50 & $0.667 \pm 0.004$ & $0.707 \pm 0.013$ & $0.780 \pm 0.004$ & $0.738 \pm 0.012$ & $0.553 \pm 0.016$ & $0.619 \pm 0.014$ & $9.210$ \\
 & 0.75 & $0.684 \pm 0.004$ & $0.788 \pm 0.011$ & $0.785 \pm 0.004$ & $0.755 \pm 0.012$ & $0.552 \pm 0.016$ & $0.643 \pm 0.014$ & $8.930$ \\
 & 1.00 & $0.686 \pm 0.004$ & $0.801 \pm 0.011$ & $0.785 \pm 0.004$ & $0.759 \pm 0.012$ & $0.552 \pm 0.016$ & $0.648 \pm 0.014$ & $8.910$ \\
\midrule
\multirow{4}{*}{0.75} 
 & 0.30 & $0.504 \pm 0.004$ & $0.108 \pm 0.009$ & $0.733 \pm 0.004$ & $0.606 \pm 0.014$ & $0.503 \pm 0.016$ & $0.530 \pm 0.015$ & $12.550$ \\
 & 0.50 & $0.663 \pm 0.004$ & $0.724 \pm 0.012$ & $0.781 \pm 0.004$ & $0.736 \pm 0.012$ & $0.553 \pm 0.016$ & $0.622 \pm 0.014$ & $9.200$ \\
 & 0.75 & $0.685 \pm 0.004$ & $0.794 \pm 0.011$ & $0.785 \pm 0.004$ & $0.754 \pm 0.012$ & $0.551 \pm 0.016$ & $0.645 \pm 0.014$ & $8.930$ \\
 & 1.00 & $0.687 \pm 0.004$ & $0.798 \pm 0.011$ & $0.785 \pm 0.004$ & $0.756 \pm 0.012$ & $0.551 \pm 0.016$ & $0.646 \pm 0.014$ & $8.910$ \\
\midrule
\multirow{4}{*}{1.00 (B)} 
 & 0.30 & $0.507 \pm 0.004$ & $0.146 \pm 0.010$ & $0.732 \pm 0.004$ & $0.598 \pm 0.014$ & $0.502 \pm 0.016$ & $0.530 \pm 0.015$ & $12.550$ \\
 & 0.50 & $0.666 \pm 0.004$ & $0.720 \pm 0.012$ & $0.780 \pm 0.004$ & $0.738 \pm 0.012$ & $0.551 \pm 0.016$ & $0.620 \pm 0.014$ & $9.200$ \\
 & 0.75 & $0.685 \pm 0.004$ & $0.805 \pm 0.011$ & $0.785 \pm 0.004$ & $0.757 \pm 0.012$ & $0.551 \pm 0.016$ & $0.645 \pm 0.014$ & $8.930$ \\
 & 1.00 & $0.687 \pm 0.004$ & $0.816 \pm 0.011$ & $0.785 \pm 0.004$ & $0.755 \pm 0.012$ & $0.551 \pm 0.016$ & $0.647 \pm 0.014$ & $8.910$ \\
\bottomrule
\end{tabular}%
}
\end{table*}

\section{Empirical Evaluation and Results}

To validate the effectiveness of our proposed methods, we conduct a comprehensive empirical evaluation on a suite of standard benchmarks. This section details our experimental setup and presents a thorough analysis of the results, demonstrating the performance of AQUA both as a standalone method and in conjunction with other optimization techniques.

\subsection{Models and Benchmarks}
We evaluate our methods on two prominent open-source models: \texttt{meta-llama/Llama-3.1-8B-Instruct} \citep{grattafiori2024llama3herdmodels}, a state-of-the-art model known for its strong performance, and \texttt{OLMoE-1B-7B-Instruct} \citep{muennighoff2025olmoeopenmixtureofexpertslanguage}, a Mixture-of-Experts model, to demonstrate generalizability.

Performance is measured across a diverse suite of standard benchmarks including MMLU \citep{wang2024mmluprorobustchallengingmultitask}, GSM8K \citep{cobbe2021trainingverifierssolvemath}, HellaSwag \citep{zellers2019hellaswag}, WinoGrande \citep{sakaguchi2019winograndeadversarialwinogradschema}, ARC Challenge \citep{clark2018thinksolvedquestionanswering}, TruthfulQA \citep{figueras2025truthknowslanguageevaluating}, and WikiText-103 \citep{merity2016pointer} using the EleutherAI \texttt{lm-evaluation-harness} \citep{eval-harness}. A detailed description of each benchmark, its purpose, and the rationale for our chosen few-shot evaluation settings are provided in Appendix ~\autoref{app:benchmark_setup}.

\subsubsection*{Hyperparameter Notation}
To maintain consistency, we define our primary hyperparameter, the \textbf{AQUA Ratio ($k_{ratio}$)}, as the fraction of dimensions retained for the attention computation after projection. For instance, a $k_{ratio}$ of 0.75 means that 75\% of the dimensions with the highest magnitudes are kept.
\vspace{-0.2em}
\subsection{Standalone AQUA Performance}
Our first experiment evaluates AQUA as a direct replacement for standard attention. We apply varying levels of pruning by adjusting the $k_{ratio}$ and observe the impact on model performance. Table \ref{tab:pruning_results_percent_condensed} presents a summary of these results, with the full, unabridged table available in Appendix~\ref{app:full_standalone_results}.

\textbf{Analysis.} The results reveal a clear and graceful trade-off between computational efficiency and model performance. For Llama-3.1-8B, we observe a remarkable ``sweet spot'': retaining 75\% of the dimensions ($k_{ratio}=0.75$) results in a negligible performance drop across all benchmarks. This demonstrates that a 25\% reduction in the computational cost of the attention dot product can be achieved with virtually no loss in model quality.

Interestingly, the OLMoE model, which uses standard Multi-Head Attention (MHA), exhibits a more gradual performance degradation compared to the GQA-based Llama-3.1. This can be attributed to the architectural differences; in GQA, a single key must retain information for multiple queries, leading to denser, less sparse key vectors as observed in our earlier analysis. Consequently, pruning dimensions from these dense keys has a more pronounced effect. In contrast, the dedicated keys in MHA are naturally sparser, making them more resilient to pruning. As pruning becomes more aggressive ($k_{ratio} \leq 0.50$), performance on both models begins to degrade, particularly on complex reasoning tasks like GSM8K, before collapsing at very low ratios.

\begin{table*}[t!]
\centering
\caption{
    Performance of \texttt{Llama-3.1-8B-Instruct} with the \textbf{AQUA-Memory} mechanism on key benchmarks. The full results, including additional benchmarks, are available in Appendix~\ref{app:full_aqua_memory_results}. Baseline is in \textbf{bold}.
}
\label{tab:qma-memory-results-main}
\small 
\begin{tabular}{@{}lcccccccc@{}}
\toprule
& \multicolumn{3}{c}{Hyperparameters} & \multicolumn{4}{c}{Key Benchmark Performance} \\
\cmidrule(lr){2-4} \cmidrule(l){5-8}
Attn. Type & $s_{ratio}$ & $k_{ratio}$ & $E_{ratio}$ & MMLU & GSM8K & HellaSwag & WikiText \\
& & & & (acc $\uparrow$) & (acc $\uparrow$) & (acc $\uparrow$) & (ppl $\downarrow$) \\ 
\midrule
Full Attn. & --- & --- & 1.000 & $\mathbf{0.687 \pm 0.004}$ & $\mathbf{0.816 \pm 0.011}$ & $\mathbf{0.785 \pm 0.004}$ & $\mathbf{8.910}$ \\
\midrule
\multirow{6}{*}{\shortstack{AQUA+\\Memory}} 
 & \multirow{3}{*}{0.10} & 0.75 & 0.675 & $0.669 \pm 0.004$ & $0.737 \pm 0.012$ & $0.781 \pm 0.004$ & $9.140$ \\
 & & 0.90 & 0.810 & $0.674 \pm 0.004$ & $0.748 \pm 0.012$ & $0.780 \pm 0.004$ & $9.100$ \\
 & & 1.00 & 0.900 & $0.675 \pm 0.004$ & $0.756 \pm 0.012$ & $0.781 \pm 0.004$ & $9.100$ \\
\cmidrule(l){2-8}
 & \multirow{3}{*}{0.25} & 0.75 & 0.563 & $0.602 \pm 0.004$ & $0.413 \pm 0.014$ & $0.755 \pm 0.004$ & $10.200$ \\
 & & 0.90 & 0.675 & $0.610 \pm 0.004$ & $0.433 \pm 0.014$ & $0.755 \pm 0.004$ & $10.090$ \\
 & & 1.00 & 0.750 & $0.609 \pm 0.004$ & $0.438 \pm 0.014$ & $0.755 \pm 0.004$ & $10.080$ \\
\bottomrule
\end{tabular}
\end{table*}

\vspace{-0.4em}
\subsection{Synergy with Token Eviction: AQUA-H2O}
A key claim of our work is that AQUA serves as a general-purpose accelerator for other KV-cache optimization techniques. To demonstrate this, we integrate AQUA with H2O, a prominent token eviction strategy. H2O identifies and retains a budget of ``Heavy Hitter'' tokens based on their accumulated attention scores. In a standard implementation, this requires computing the full attention matrix first. In our hybrid \textbf{AQUA-H2O} approach, we first use AQUA to compute an \textit{approximate} attention score matrix very efficiently. These approximate scores are then used to identify the Heavy Hitters for H2O's eviction policy \citep{zhang2023h2oheavyhitteroracleefficient}, thus accelerating the eviction process. So, in the same way we can integrate AQUA with any token eviction methods to accelerate their compute efficiency.

\textbf{Analysis.} The results for combining H2O token eviction with AQUA pruning are demonstrated in Table \ref{tab:qma-h2o_results_percent}. The configuration where $H2O_{ratio}=1.00$ is equivalent to the standalone AQUA model, serving as our baseline. The most compelling results emerge when aggressive token eviction is paired with moderate AQUA pruning. For instance, with an $H2O_{ratio}$ of $0.50$ (evicting half the tokens) and a $k_{ratio}$ of $0.75$, the model's performance across all benchmarks is nearly identical to that of the full, unmodified baseline. This powerfully demonstrates that we can achieve massive reduction in latency and computation (from both eviction and AQUA) while preserving model performance. The results for the \texttt{OLMoE-1B-7B-Instruct} model, which show a similar trend, are available in Appendix~\ref{app:qma_h2o_olmoe}.

\subsection{Combined Compute and Memory Savings: AQUA-Memory}
Finally, we introduce \textbf{AQUA-Memory}, a variant of our method designed to directly and simultaneously reduce both KV-cache memory and computational load. This approach employs a two-stage pruning strategy:

\begin{enumerate}
    \item \textbf{Static Memory Pruning:} First, we permanently discard a fraction of the dimensions corresponding to the lowest-importance principal components (i.e., the last columns of the projection matrix $P$) before the key and value vectors are written to the KV-cache. This yields a direct and predictable saving in memory, controlled by a slice ratio ($S_{ratio}$) hyperparameter representing the fraction of dimensions removed.

    \item \textbf{Dynamic Compute Pruning:} On the remaining, reduced-dimension vectors, we then apply our standard dynamic magnitude selection. The $k_{ratio}$ hyperparameter is applied to this smaller set of dimensions to further reduce the computational cost of the attention dot product.
\end{enumerate}

The total effective reduction is captured by the \textbf{Effective Ratio ($E_{ratio}$)}, which represents the final fraction of the original dimensions used in the attention calculation. It is defined as $ Eff_{ratio} = (1 - slice_{ratio}) \times k_{\text{ratio}} $
This dual-pruning mechanism provides a powerful framework for navigating the trade-off between memory, compute, and model performance.

\textbf{Analysis.} Table \ref{tab:qma-memory-results-main} presents the results of this direct memory and compute trade-off. By slicing off just 10\% of the dimensions before caching ($slice_{ratio}=0.10$), we achieve a 10\% reduction in KV-cache memory. When we then apply a gentle compute reduction on the remaining dimensions ($k_{ratio}=0.90$), the performance drop is minimal, with perplexity only increasing to 9.10 from 8.91. This result is highly significant, as it provides a direct, controllable method for reducing the KV-cache size with a predictable and graceful degradation in performance. As expected, a more aggressive memory slice ($slice_{ratio}=0.25$) leads to a more substantial performance hit, establishing a clear boundary for this technique.

\section{Conclusion}

In this work, we addressed the critical efficiency bottleneck of the Transformer attention mechanism by introducing AQUA, a novel approximation strategy that reduces computational and memory load. Our approach is centered on a simple yet powerful insight: by projecting query and key vectors into a new coordinate space, we can dynamically select a small subset of the most salient dimensions based on their magnitude, achieving significant efficiency gains with a remarkably graceful performance trade-off.

We have demonstrated that our method, using an offline-calibrated and language-agnostic projection matrix, can reduce the core attention computation by 25\% with negligible impact on performance across a wide range of standard benchmarks. Furthermore, we have shown its versatility, proving it can function effectively as a standalone mechanism, as a computational accelerator for existing token eviction strategies like H2O \citep{zhang2023h2oheavyhitteroracleefficient}, and as a direct method for reducing KV-cache memory. Our theoretical analysis provides a clear understanding of the computational break-even point, confirming that the benefits of our method grow with sequence length.

The primary trade-off of our approach is the initial projection overhead, which makes it most suitable for the long-sequence regimes where attention optimization is most critical. The retention ratio, $k_{ratio}$, serves as a controllable hyperparameter, empowering practitioners to tune the balance between efficiency and accuracy to fit their specific application needs. This work opens several exciting avenues for future research, most notably the development of adaptive methods that could learn to set this ratio dynamically based on the context. Further exploration into applying the AQUA framework to other modalities, such as Vision Transformers, and combining it with complementary techniques like quantization, promises to further broaden its impact, making powerful models more efficient and accessible.

\newpage

\bibliography{main}
\bibliographystyle{tmlr}

\newpage

\appendix

\section{Appendix}

\subsection{A Broader Survey of Related Work}
\label{app:related_work_survey}

The challenge of optimizing Large Language Model (LLM) inference has spanned a wide array of research, primarily focused on mitigating the quadratic complexity of the attention mechanism. While some approaches have targeted other architectural components, such as pruning MLP layers \citep{kiefer2024comparativestudypruningmethods, ding2025adaptivepruningpretrainedtransformer}, our focus lies with methods that address the attention bottleneck. These works can be broadly categorized into quantization techniques, attention approximation methods, and token eviction strategies, each targeting a different layer of the efficiency problem.

\subsection*{Quantization Techniques}
\textbf{Quantization} is a well-established method for model compression that reduces the numerical precision of model weights, activations, or the KV-cache itself. Works such as KIVI \citep{10.5555/3692070.3693381} and KV-Quant \citep{hooper2025kvquant10millioncontext} have demonstrated that the precision of the Key and Value matrices can be reduced to as low as 2-bits with minimal performance degradation. These methods are largely orthogonal to the structural changes proposed in other works and can often be applied in conjunction with them to achieve cumulative efficiency gains.

\subsection*{Attention Approximation Techniques}
This line of research seeks to reduce computational cost by approximating the attention mechanism, rather than computing it in its entirety. These methods often leverage the observation that the attention matrix or its constituent components exhibit low-rank or sparse properties.

\textbf{Low-Rank Matrix Approximations.} Early works like Linformer \citep{wang2020linformerselfattentionlinearcomplexity} and Reformer \citep{kitaev2020reformerefficienttransformer} established that the attention score matrix is often low-rank or can be approximated using locality-sensitive hashing to reduce complexity from quadratic to near-linear. These methods typically require architectural changes and retraining, making them less applicable to pre-trained models.

\textbf{KV-Cache Dimensionality Reduction.} A more recent approach focuses on compressing the Key and Value vectors within the KV-cache. A notable example is EigenAttention \citep{saxena2024eigenattentionattentionlowrank}, which posits that the K and V activations lie in a low-dimensional subspace. It compresses the KV-cache by decomposing the projection weights for K and V into low-rank factors, thereby reducing the stored dimension $d_{head}$. While effective at reducing memory, this approach has two primary limitations. First, the compression rank is a fixed hyperparameter that must be decided offline; it cannot be dynamically adjusted at runtime. Second, the method does not provide a theoretical bound on its performance trade-offs. Our work is similar in its goal of reducing dimensionality but differs by operating on projected vectors at runtime and providing a formal analysis of its computational benefits.

\subsection*{Token Eviction Techniques}
A popular and effective strategy for managing long contexts is to prune the KV-cache by evicting tokens deemed less important. This directly reduces both memory usage and the computational cost of the attention calculation.

\textbf{Token Eviction.} These methods identify and permanently discard tokens from the KV-cache. A seminal work in this area, H2O \citep{zhang2023h2oheavyhitteroracleefficient}, identifies ``Heavy Hitter'' tokens by monitoring their accumulated attention scores over time. Its eviction policy dynamically retains a balance of these important H2 tokens alongside the most recent tokens. While highly effective at reducing the memory and compute footprint, this approach risks the permanent loss of information, which can lead to a non-trivial degradation in model quality if important, but not immediately ``heavy-hitting'' tokens are discarded.

\textbf{Hybrid and Magnitude-Based Approaches.} More recent works have explored more nuanced strategies that combine approximation with token selection. SparQ Attention \citep{ribar2024sparqattentionbandwidthefficientllm} uses the high-magnitude dimensions of a query to compute approximate attention scores, which are then used to select a subset of ``top'' keys. Full-rank attention is then computed for only this subset. While conceptually similar to our approach in its use of query magnitudes, SparQ \citep{ribar2024sparqattentionbandwidthefficientllm} has notable drawbacks: it requires costly non-contiguous column indexing of the key vectors and stores two copies of past keys to maintain efficiency, increasing the memory footprint by 50\%.

Another related work, LoKi Attention \citep{singhania2024lokilowrankkeysefficient}, uses an offline-computed projection matrix to transform query and key vectors and then truncates them by simply slicing off the trailing dimensions. Based on the resulting approximate attention scores, it temporarily drops tokens for the subsequent computation. However, LoKi does not permanently evict tokens from the cache, thereby forgoing memory savings for the sake of preserving information. Crucially, its strategy of statically slicing the first few dimensions after projection differs from our dynamic, magnitude-based selection, a distinction we empirically justify in Section 4.

\subsection{A Primer on SVD for PCA}
Singular Value Decomposition factorizes any matrix $A \in \mathbb{R}^{m \times n}$ into three matrices:
\begin{equation}
    A = U \Sigma V^T
\end{equation}
where $U \in \mathbb{R}^{m \times m}$ and $V \in \mathbb{R}^{n \times n}$ are orthogonal matrices, and $\Sigma \in \mathbb{R}^{m \times n}$ is a diagonal matrix containing the singular values. The columns of $V$ are the principal components (or principal directions) of the row-space of $A$. Projecting the rows of $A$ onto the first few columns of $V$ concentrates the maximum possible variance (i.e., ``energy'' or ``information'') into those new dimensions.

\subsection{Derivation of SVD Computational Complexity}
\label{app:svd_proof}

In our analysis, we state that the computational complexity of performing a full Singular Value Decomposition (SVD) on the key cache matrix $K_{:i+1}$ is $O(\min((i+1)d_{head}^2, (i+1)^2d_{head}))$. This appendix provides a brief derivation for this standard result from numerical linear golub 2013 matrix
\citep{Golub_2013_MC}.

Let the matrix in question be $A \in \mathbb{R}^{m \times n}$, where, in our context, $m = i+1$ (the sequence length) and $n = d_{head}$ (the head dimension). The SVD of $A$ is given by $A = U\Sigma V^T$. Standard algorithms for computing the full SVD typically rely on first finding the eigenvalues and eigenvectors of either $A^T A$ or $A A^T$. The choice between these two paths depends on which of the two matrices is smaller, as that determines the more efficient route.

\subsubsection*{Path 1: Eigendecomposition of $A^T A$}
This path is generally preferred when $m \ge n$ (i.e., when the sequence length is greater than or equal to the head dimension, which is the common case in LLMs).

\begin{enumerate}
    \item \textbf{Form the covariance matrix $A^T A$:}
    \begin{itemize}
        \item The dimensions are $(n \times m) \times (m \times n) = (n \times n)$.
        \item The computational cost of this matrix multiplication is $O(n^2 m)$.
    \end{itemize}
    \item \textbf{Perform eigendecomposition on $A^T A$:}
    \begin{itemize}
        \item We solve $(A^T A)V = V\Lambda$, where $\Lambda$ is the diagonal matrix of eigenvalues and the columns of $V$ are the eigenvectors.
        \item The cost of eigendecomposition for an $n \times n$ matrix is typically $O(n^3)$.
    \end{itemize}
\end{enumerate}
The total complexity for this path is the sum of these steps, $O(n^2 m + n^3)$. Since we assume $m \ge n$, the dominant term is $O(n^2 m)$. Substituting our original variable names, this complexity is $O(d_{head}^2 (i+1))$.

\subsubsection*{Path 2: Eigendecomposition of $A A^T$}
This path is more efficient when $n > m$ (i.e., when the head dimension is larger than the sequence length, a less common scenario).

\begin{enumerate}
    \item \textbf{Form the covariance matrix $A A^T$:}
    \begin{itemize}
        \item The dimensions are $(m \times n) \times (n \times m) = (m \times m)$.
        \item The computational cost of this matrix multiplication is $O(m^2 n)$.
    \end{itemize}
    \item \textbf{Perform eigendecomposition on $A A^T$:}
    \begin{itemize}
        \item We solve $(A A^T)U = U\Lambda$.
        \item The cost of eigendecomposition for an $m \times m$ matrix is typically $O(m^3)$.
    \end{itemize}
\end{enumerate}
The total complexity for this path is $O(m^2 n + m^3)$. Since we assume $n > m$, the dominant term is $O(m^2 n)$. Substituting our original variable names, this complexity is $O((i+1)^2 d_{head})$.

\subsubsection*{Overall Complexity}
An efficient SVD algorithm will internally choose the more optimal of these two paths based on the matrix dimensions. Therefore, the overall computational complexity is the minimum of the dominant costs from each path.
$$ \text{Complexity} = O(\min(n^2 m, m^2 n)) $$
Substituting $m = i+1$ and $n = d_{head}$, we arrive at the complexity cited in the main text:
$$ \text{Complexity} = O(\min(d_{head}^2 (i+1), (i+1)^2 d_{head})) $$

\begin{figure*}[!t]
\centering
\includegraphics[width=1.0\textwidth]{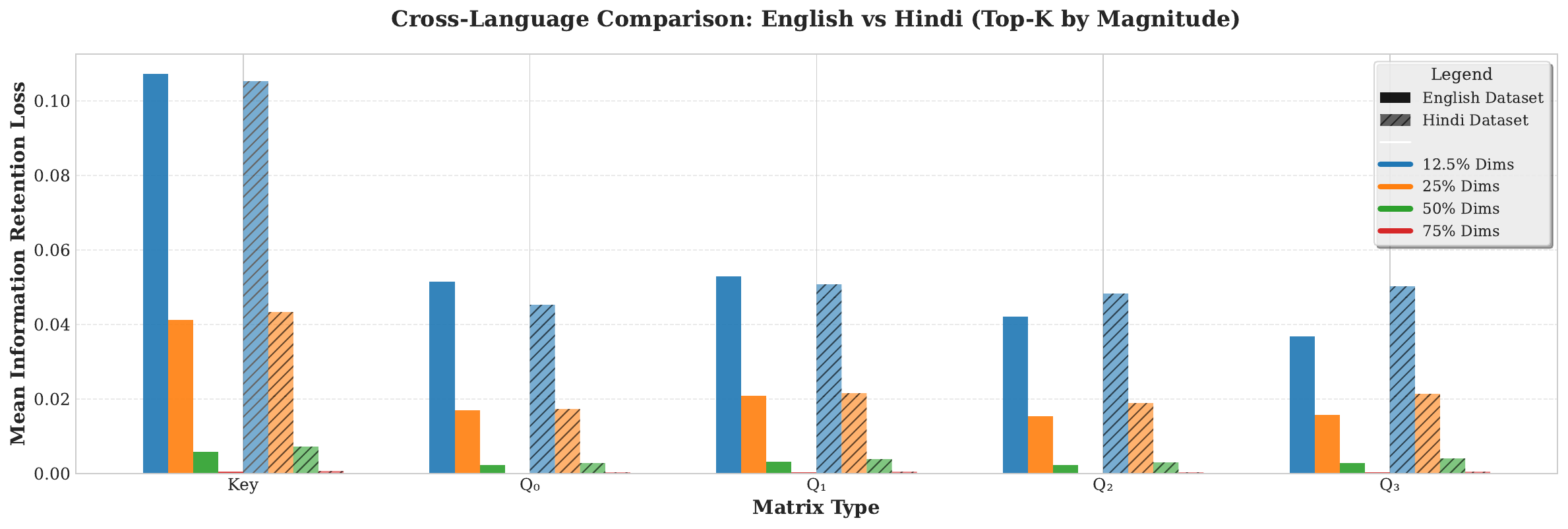}
\caption{Cross-lingual comparison of mean information retention loss using an English-calibrated projection matrix on English (\texttt{WikiText}) and Hindi (\texttt{wikipedia-hi}) datasets for a full GQA group (1 Key and 4 corresponding Query matrices). The similar loss profiles across all query and key matrices demonstrate that the projection matrix generalizes well across languages with different scripts, indicating it captures fundamental, language-agnostic properties of the attention heads.}
\label{fig:language_comparison_full}
\vspace{-1em}
\end{figure*}

\subsection{Detailed Theoretical Analysis and Proofs}
\label{app:detailed_proofs}

This appendix provides the detailed proofs for the complexity results presented in Section~\ref{sec:analysis_results}. The analysis focuses on the computational cost required to produce the unnormalized attention scores, as this is the primary stage targeted by our optimization.

\begin{proposition}[Complexity of Standard Attention]
\label{prop:std_attention}
The computational complexity of the unnormalized score calculation in standard auto-regressive attention for a single head at token $i+1$ is $C_{std} = O((i+1) \cdot d_{head})$ \citep{vaswani2023attentionneed, tay2022efficienttransformerssurvey}.
\end{proposition}
\begin{proof}
The cost is dominated by the matrix-vector product $q_i K_{:i+1}^T$, where the query $q_i \in \mathbb{R}^{1 \times d_{head}}$ and the transposed key cache $K_{:i+1}^T \in \mathbb{R}^{d_{head} \times (i+1)}$. This operation requires $(i+1) \cdot d_{head}$ multiplication-addition pairs, leading to a complexity of $O((i+1) \cdot d_{head})$ \citep{Goodfellow-et-al-2016-DL}.
\end{proof}

\begin{theorem}[Complexity of AQUA]
\label{thm:aqua_complexity}
The computational complexity of the unnormalized score calculation in the AQUA algorithm (Algorithm~\ref{alg:query_mag_detailed}) for a single head at token $i+1$ is $C_{AQUA} = O(d_{head}^2 + (i+1) \cdot k)$.
\end{theorem}
\begin{proof}
The total complexity is the sum of the complexities of the constituent steps of the algorithm:
\begin{enumerate}
    \item \textbf{Projections:} Projecting the current query $q_i$ and key $k_i$ with matrix $P \in \mathbb{R}^{d_{head} \times d_{head}}$ requires two matrix-vector multiplications, for a total cost of $O(d_{head}^2)$ \citep{Goodfellow-et-al-2016-DL}.
    \item \textbf{Magnitude Calculation \& Top-k Selection:} Computing the element-wise absolute value of $\hat{q}_i$ is an $O(d_{head})$ operation. Finding the indices of the top $k$ elements can be done efficiently using a selection algorithm (e.g., Introselect) in $O(d_{head})$ average-case time \citep{BLUM1973448}.
    \item \textbf{Dimension Slicing:} This memory operation to create the sliced key matrix $\tilde{K}_{:i+1}$ requires accessing $(i+1) \cdot k$ elements, with a computational cost bounded by $O((i+1) \cdot k)$.
    \item \textbf{Approximate Attention Computation:} The final matrix-vector product is between $\tilde{q}_i \in \mathbb{R}^{1 \times k}$ and $\tilde{K}_{:i+1}^T \in \mathbb{R}^{k \times (i+1)}$. The complexity of this step is $O((i+1) \cdot k)$.
\end{enumerate}
Summing these complexities and retaining the dominant terms, we get:
\begin{align*}
C_{AQUA} &= O(d_{\text{head}}^2) + O(d_{\text{head}}) + O((i+1) \cdot k) \\
       &= O(d_{head}^2 + (i+1) \cdot k)
\end{align*}
This concludes the proof.
\end{proof}

\begin{corollary}[Computational Break-Even Point]
\label{cor:break_even}
The AQUA algorithm is computationally more efficient than standard attention for computing unnormalized scores when the sequence length $i+1$ satisfies the condition: $i+1 > \frac{d_{head}^2}{d_{head} - k}$.
\end{corollary}
\begin{proof}
We find the condition for which $C_{AQUA} < C_{std}$:
\begin{align*}
d_{head}^2 + (i+1) \cdot k &< (i+1) \cdot d_{head} \\
d_{head}^2 &< (i+1) \cdot d_{head} - (i+1) \cdot k \\
d_{head}^2 &< (i+1)(d_{head} - k)
\end{align*}
Assuming $k < d_{head}$ (the practical use case for approximation), we can divide by the positive term $(d_{head} - k)$:
$$ \frac{d_{head}^2}{d_{head} - k} < i+1 $$
This proves the corollary.
\end{proof}

Let us clarify this relationship between the hyperparameter $k$ and current sequence length $i$ bydissecting the trade-off between fixed overhead and accumulated savings.

\subsubsection*{\textbf{Decomposing Computational Costs}}
To understand the break-even point, we must separate the costs into two components:
\begin{enumerate}
    \item \textbf{Fixed Overhead Cost ($C_{fixed}$):} This is a one-time cost incurred at each step, independent of the sequence length $i$. It is dominated by the projection of the current query and key vectors, with complexity $O(d_{head}^2)$. This cost is paid regardless of whether the context is short or long.

    \item \textbf{Variable Savings ($S_{var}$):} This represents the computational savings achieved for each of the $i+1$ tokens in the context. The saving for each token is proportional to the number of dimensions we prune, $(d_{head} - k)$. Thus, the total accumulated savings across the entire context is proportional to $(i+1)(d_{head} - k)$.
\end{enumerate}

\subsubsection*{\textbf{The Break-Even Condition and the Role of $k$}}
The algorithm becomes computationally superior precisely when the total accumulated savings surpass the fixed overhead cost:
$$ (i+1)(d_{head} - k) > d_{head}^2 $$
This inequality shows that when $k$ is small (an aggressive approximation), the per-token saving $(d_{head} - k)$ is large, meaning a shorter sequence is needed to ``pay off'' the fixed overhead. Conversely, when $k$ is large, the per-token saving is small, and a much longer sequence is required to accumulate enough savings to overcome the same fixed cost.

\subsubsection*{A Numerical Example}
Let us consider a typical head dimension, $d_{head} = 128$, making the fixed overhead proportional to $d_{head}^2 = 16,384$.
\begin{itemize}
    \item \textbf{Case 1: Aggressive Approximation ($k=16$)}
    The per-token saving is proportional to $128 - 16 = 112$. The break-even point is $i+1 > \frac{16384}{112} \approx 147$ tokens.

    \item \textbf{Case 2: Moderate Approximation ($k=64$)}
    The per-token saving is proportional to $128 - 64 = 64$. The break-even point is $i+1 > \frac{16384}{64} = 256$ tokens.

    \item \textbf{Case 3: Fine Approximation ($k=112$)}
    The per-token saving is proportional to $128 - 112 = 16$. The break-even point is $i+1 > \frac{16384}{16} = 1024$ tokens.
    
    \item \textbf{Case 4: No Approximation ($k=d_{head}=128$)}
    In this edge case, the per-token saving is $128 - 128 = 0$. The break-even condition becomes $i+1 > \frac{16384}{0}$, which approaches infinity. This confirms the intuition: if no dimensions are pruned, there are no computational savings. Because an additional fixed overhead is incurred at every decoding step for the projections, the AQUA method will always be less efficient than standard attention in this scenario.
\end{itemize}
This example clearly illustrates the principle: a more aggressive approximation (smaller $k$) yields a greater per-token saving, thus requiring a shorter sequence to realize a net computational gain.

\begin{figure*}[t]
\centering
\includegraphics[width=1.0\textwidth]{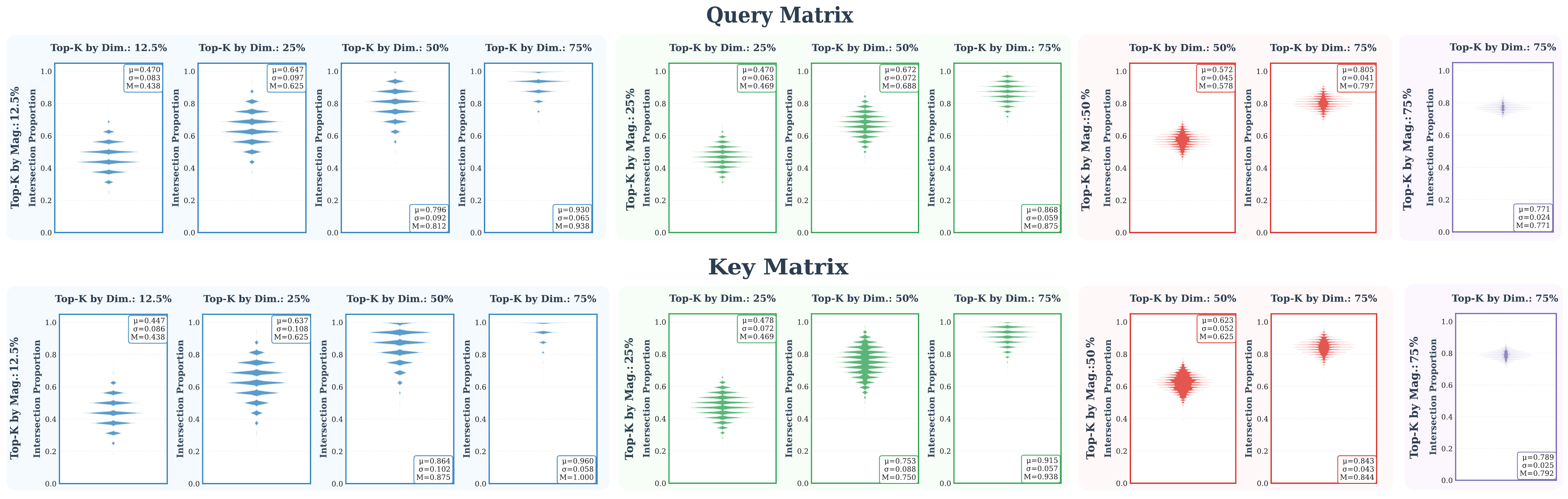}
\caption{Overlap analysis for the Query and Key matrices (Layer 31, Head 31). The plots show the intersection proportion between the set of top dimensions selected by magnitude and the set of top dimensions selected by PCA index. The low overlap, especially for smaller K, demonstrates that global variance (PCA) and token-specific activity (magnitude) are not equivalent notions of importance.}
\label{fig:key_query_overlap}
\end{figure*}

\subsection{Detailed Cross-Lingual Generalizability Analysis}
\label{app:detailed_cross_linguality}

This appendix provides a more detailed account of the cross-lingual experiment designed to test the robustness of our offline projection matrix.

\subsubsection*{Experimental Design and Rationale}
To create a rigorous test case, we sought a language that was maximally distant from the English used in our calibration dataset (\texttt{BookCorpus}). We selected Hindi for two primary reasons. First, it is an officially supported language in the \texttt{meta-llama/Llama-3.1-8B-Instruct} model, ensuring that the tokenizer and internal representations are well-defined. Second, and more importantly, Hindi uses the Devanagari script, which is structurally and visually unrelated to the Latin script used for English. This provides an extreme test: if the projection matrix had overfit to linguistic or orthographic features of English, we would expect its performance to degrade significantly when applied to Hindi activations.

We ensured experimental consistency by sampling an equal number of query and key vectors from both the English evaluation set (\texttt{WikiText} \citep{merity2016pointer}) and the Hindi dataset (\texttt{wikipedia-hi}).

\subsubsection*{Results and Discussion}
The full results for the entire GQA group are presented in Figure \ref{fig:language_comparison_full}. The plot shows the mean information retention loss for the shared Key matrix and all four associated Query matrices ($Q_0$ through $Q_3$). These are the matrices corresponding to a single projection matrix. Across all matrix types, the performance on the Hindi data closely mirrors the performance on the English data. This consistency strongly supports our hypothesis that the SVD-based projection method captures fundamental, language-agnostic structural properties of the attention heads, rather than superficial linguistic patterns. This makes our approach highly generalizable and robust for multilingual applications.

\subsection{Detailed Analysis of Magnitude vs. PCA-based Selection}
\label{app:magnitude_vs_pca}

This appendix provides the detailed empirical analysis that justifies our use of magnitude-based dimension selection over a naive slicing of principal components.

\subsubsection*{\textbf{Methodology}}
Our study investigates the overlap between two distinct sets of ``important'' dimensions within the Key and Query matrices of a specific attention head (Layer 31, Head 31) from the \texttt{meta-llama/Llama-3.1-8B-Instruct} \citep{grattafiori2024llama3herdmodels} model. The analysis was conducted on text sequences from the WikiText-2-raw-v1 dataset. Let $v \in \mathbb{R}^{d_{head}}$ be a vector from either a Query or Key matrix. We define two sets of indices from the full set of dimensions $\mathcal{I} = \{1, \dots, d_{head}\}$:
\begin{enumerate}
    \item \textbf{The Set of Top-$K$ Magnitude Dimensions}, $\mathcal{S}_{mag}(v, K)$, is the set of indices corresponding to the $K$ largest absolute values of the components of the unprojected vector $v$. Formally:
    $$ \mathcal{S}_{mag}(v, K) = \arg \text{TopK}_{j \in \mathcal{I}} (|v_j|), \quad \text{where } |\mathcal{S}_{mag}(v, K)| = K $$
    \item \textbf{The Set of Top-$K'$ PCA Dimensions}, $\mathcal{S}_{pca}(K')$, is the set of indices corresponding to the first $K'$ principal components derived from our offline calibration. As these components are ordered by variance, this is simply the set of the first $K'$ indices:
    $$ \mathcal{S}_{pca}(K') = \{1, 2, \dots, K'\} $$
\end{enumerate}
We consider values for $K$ and $K'$ from the set $\{0.125 \cdot d_{head}, 0.25 \cdot d_{head}, 0.5 \cdot d_{head}, 0.75 \cdot d_{head}\}$. To quantify the alignment between token-specific importance (magnitude) and global importance (PCA), we calculate the intersection proportion, $\rho$, for each vector $v$:
$$ \rho(v, K, K') = \frac{|\mathcal{S}_{mag}(v, K) \cap \mathcal{S}_{pca}(K')|}{K} $$
This metric measures the fraction of the top magnitude dimensions that are also captured by the top $K'$ principal components. The distributions of $\rho$ across the dataset are then visualized using violin plots.

\subsubsection*{\textbf{Results and Discussion}}
The results are presented in Figure \ref{fig:key_query_overlap}. The key observations are as follows:
\begin{itemize}
    \item \textbf{Discrepancy Between Magnitude and PCA Importance:} The central finding is that the overlap is often far from perfect. For example, when selecting the top 12.5\% of dimensions by magnitude, only a fraction of these are captured by the top 12.5\% of principal components. This indicates a significant discrepancy between the dimensions that are most ``active'' for a specific token (high magnitude) and those that capture the most variance globally (top principal components).
    \item \textbf{Increasing Overlap with More PCA Components:} As expected, the overlap increases as we include more principal components (moving horizontally across the columns in the figures). However, even when considering 75\% of the PCA dimensions, the overlap with the top 12.5\% of magnitude dimensions is not total.
\end{itemize}

This empirical analysis demonstrates that simply slicing the top $k$ dimensions by index after projection is a suboptimal strategy. While PCA effectively identifies directions of maximal global variance, these directions do not consistently align with the dimensions that are most salient for a specific query, as indicated by their magnitude. This finding provides strong motivation for the central mechanism of the AQUA algorithm: dynamically selecting the dimensions with the highest magnitude for each query, rather than relying on a fixed, static set of principal component indices.

\subsection{Rotational Invariance of Attention Scores}
\label{app:rotational_invariance}

This appendix provides the formal proof for the claim that projecting query and key vectors with an orthogonal matrix is a lossless rotation that preserves the dot product scores.

\begin{lemma}[Rotational Invariance of Attention Scores]
Let $P \in \mathbb{R}^{d_{head} \times d_{head}}$ be an orthogonal projection matrix (i.e., $PP^T = P^T P = I$) derived from offline calibration. Let the projected query and key matrices be $\hat{q}_i = q_i P$ and $\hat{K}_{:i+1} = K_{:i+1} P$. The attention scores computed using the original vectors are identical to those computed using the projected vectors.
\end{lemma}
\begin{proof}
The original attention scores are given by the dot product $S = q_i K_{:i+1}^T$. The attention scores computed with the projected vectors are $\hat{S} = \hat{q}_i \hat{K}_{:i+1}^T$.

We can show that $\hat{S}$ is equivalent to $S$:
\begin{align*}
    \hat{S} &= (\hat{q}_i) (\hat{K}_{:i+1})^T \\
            &= (q_i P) (K_{:i+1} P)^T  && \text{[Substituting definitions]} \\
            &= (q_i P) (P^T K_{:i+1}^T) && \text{[Using the transpose property $(AB)^T = B^T A^T$]} \\
            &= q_i (P P^T) K_{:i+1}^T  && \text{[Associativity of matrix multiplication]} \\
            &= q_i I K_{:i+1}^T        && \text{[Since P is orthogonal, $PP^T = I$]} \\
            &= q_i K_{:i+1}^T = S
\end{align*}
Thus, the scores are identical ($\hat{S} = S$). This proves that the projection is a lossless rotation of the coordinate space that preserves the dot product relationships between all query and key vectors. The only source of approximation error in our method, therefore, comes from the subsequent step of selecting a subset of the dimensions.
\end{proof}

\subsection{Benchmark and Evaluation Details}
\label{app:benchmark_setup}

This appendix provides a detailed account of the models and benchmark suite used in our empirical evaluation.

\subsubsection*{Models}
\begin{itemize}
    \item \textbf{\texttt{meta-llama/Llama-3.1-8B-Instruct} \citep{grattafiori2024llama3herdmodels}:} A powerful, state-of-the-art 8-billion parameter model that serves as our primary testbed. It features a model dimension ($d_{model}$) of 4096 across 32 layers. The model is built on a Grouped-Query Attention (GQA) architecture, with 32 query heads and 8 key/value heads (a group size of 4), where each head has a dimension ($d_{head}$) of 128.

    \item \textbf{\texttt{OLMoE-1B-7B-Instruct} \citep{muennighoff2025olmoeopenmixtureofexpertslanguage}:} A Mixture-of-Experts (MoE) model whose inclusion allows us to test the generalizability of our method on a different architecture type. This model is built on a standard Multi-Head Attention (MHA) architecture, where every head has a unique Key and Query. It features a model dimension ($d_{model}$) of 4096, 64 experts, 27 layers, 16 heads, and a head dimension ($d_{head}$) of 128.
\end{itemize}

\begin{table*}[t!]
\centering
\caption{
    Full performance results for \texttt{Llama-3.1-8B-Instruct} and \texttt{OLMoE-1B-7B-Instruct} on various benchmarks under different levels of AQUA pruning ratio ($k_{ratio}$). The baseline for each model, denoted by `B' ($k_{ratio}=1.0$), is highlighted in bold. Further details are provided in the Appendix.
}
\label{tab:pruning_results_percent_full}
\resizebox{\textwidth}{!}{%
\begin{tabular}{@{}llccccccc@{}}
\toprule
\multirow{2}{*}{Model} & \multirow{2}{*}{\shortstack{$k_{ratio}$}} & MMLU & GSM8K & HellaSwag & WinoGrande & \shortstack{TruthfulQA\\MC2} & \shortstack{ARC\\Challenge} & WikiText \\
\cmidrule(l){3-9} 
& & (acc $\uparrow$) & (acc $\uparrow$) & (acc $\uparrow$) & (acc $\uparrow$) & (acc $\uparrow$) & (acc $\uparrow$) & (ppl $\downarrow$) \\ 
\midrule
\multirow{8}{*}{\shortstack{Llama-3.1-8B \\ Instruct}} 
 & \textbf{B} & $\bm{0.687 \pm 0.004}$ & $\bm{0.816 \pm 0.011}$ & $\bm{0.785 \pm 0.004}$ & $\bm{0.755 \pm 0.012}$ & $\bm{0.551 \pm 0.016}$ & $\bm{0.647 \pm 0.014}$ & $\bm{8.910}$ \\
 & 0.90  & $0.687 \pm 0.004$ & $0.792 \pm 0.011$ & $0.784 \pm 0.004$ & $0.756 \pm 0.012$ & $0.551 \pm 0.016$ & $0.647 \pm 0.014$ & $8.910$ \\
 & 0.75  & $0.685 \pm 0.004$ & $0.805 \pm 0.011$ & $0.785 \pm 0.004$ & $0.757 \pm 0.012$ & $0.551 \pm 0.016$ & $0.645 \pm 0.014$ & $8.930$ \\
 & 0.50  & $0.666 \pm 0.004$ & $0.720 \pm 0.012$ & $0.780 \pm 0.004$ & $0.738 \pm 0.012$ & $0.551 \pm 0.016$ & $0.620 \pm 0.014$ & $9.200$ \\
 & 0.40  & $0.634 \pm 0.004$ & $0.541 \pm 0.014$ & $0.773 \pm 0.004$ & $0.696 \pm 0.013$ & $0.540 \pm 0.016$ & $0.600 \pm 0.014$ & $9.810$ \\
 & 0.30  & $0.507 \pm 0.004$ & $0.146 \pm 0.010$ & $0.732 \pm 0.004$ & $0.598 \pm 0.014$ & $0.502 \pm 0.016$ & $0.530 \pm 0.015$ & $12.550$ \\
 & 0.20  & $0.242 \pm 0.004$ & $0.019 \pm 0.004$ & $0.391 \pm 0.005$ & $0.511 \pm 0.014$ & $0.471 \pm 0.015$ & $0.236 \pm 0.012$ & $44.960$ \\
 & 0.10  & $0.230 \pm 0.004$ & $0.012 \pm 0.003$ & $0.261 \pm 0.004$ & $0.496 \pm 0.014$ & $0.491 \pm 0.016$ & $0.236 \pm 0.012$ & $970.440$ \\
\midrule
\multirow{8}{*}{\shortstack{OLMoE-1B-7B \\ Instruct}} 
 & \textbf{B} & $\bm{0.530 \pm 0.004}$ & $\bm{0.451 \pm 0.014}$ & $\bm{0.783 \pm 0.004}$ & $\bm{0.673 \pm 0.013}$ & $\bm{0.491 \pm 0.016}$ & $\bm{0.532 \pm 0.015}$ & $\bm{11.340}$ \\
 & 0.90 & $0.530 \pm 0.004$ & $0.463 \pm 0.014$ & $0.782 \pm 0.004$ & $0.668 \pm 0.013$ & $0.489 \pm 0.016$ & $0.538 \pm 0.015$ & $11.340$ \\
 & 0.75 & $0.529 \pm 0.004$ & $0.453 \pm 0.014$ & $0.783 \pm 0.004$ & $0.669 \pm 0.013$ & $0.485 \pm 0.016$ & $0.542 \pm 0.015$ & $11.330$ \\
 & 0.50 & $0.526 \pm 0.004$ & $0.426 \pm 0.014$ & $0.778 \pm 0.004$ & $0.658 \pm 0.013$ & $0.488 \pm 0.016$ & $0.540 \pm 0.015$ & $11.340$ \\
 & 0.40 & $0.512 \pm 0.004$ & $0.385 \pm 0.013$ & $0.771 \pm 0.004$ & $0.640 \pm 0.013$ & $0.488 \pm 0.016$ & $0.532 \pm 0.015$ & $11.470$ \\
 & 0.30 & $0.485 \pm 0.004$ & $0.253 \pm 0.012$ & $0.747 \pm 0.004$ & $0.615 \pm 0.014$ & $0.481 \pm 0.016$ & $0.513 \pm 0.015$ & $12.140$ \\
 & 0.20 & $0.403 \pm 0.004$ & $0.042 \pm 0.006$ & $0.665 \pm 0.005$ & $0.549 \pm 0.014$ & $0.485 \pm 0.016$ & $0.378 \pm 0.014$ & $15.960$ \\
 & 0.10 & $0.243 \pm 0.004$ & $0.014 \pm 0.003$ & $0.346 \pm 0.005$ & $0.511 \pm 0.014$ & $0.486 \pm 0.016$ & $0.239 \pm 0.012$ & $74.440$ \\
\bottomrule
\end{tabular}%
}
\end{table*}

\subsubsection*{Benchmark Suite Rationale}
All evaluations were conducted using the standardized EleutherAI \texttt{lm-evaluation-harness} \citep{eval-harness} framework. The chosen benchmarks and few-shot settings align with common practices in LLM evaluation to ensure comparability and reproducibility.

\begin{itemize}[leftmargin=*, noitemsep]
    \item \textbf{MMLU \citep{wang2024mmluprorobustchallengingmultitask} (5-shot):} This benchmark evaluates massive multitask language understanding across 57 subjects. The 5-shot setting is a standard and challenging configuration widely used for reporting performance on top-tier LLMs and public leaderboards \citep{Mai2024MMLU}.

    \item \textbf{GSM8K \citep{cobbe2021trainingverifierssolvemath} (8-shot):} This benchmark tests grade-school mathematical reasoning. We use an 8-shot Chain-of-Thought (CoT) prompting strategy, as it is the standard method for eliciting multi-step reasoning from capable models \citep{10.5555/3600270.3602070}.
    
    \item \textbf{HellaSwag \citep{zellers2019hellaswag} (10-shot):} This benchmark evaluates commonsense inference about everyday events. The 10-shot setting is commonly reported in recent literature for state-of-the-art models.
    
    \item \textbf{WinoGrande \citep{sakaguchi2019winograndeadversarialwinogradschema} (5-shot):} This benchmark targets commonsense reasoning through pronoun resolution problems. The 5-shot setting is the standard for evaluation on this task.
    
    \item \textbf{ARC Challenge \citep{clark2018thinksolvedquestionanswering} (25-shot):} The AI2 Reasoning Challenge (ARC) contains difficult science questions. The 25-shot setting is standard for recent, high-performance model evaluations.
    
    \item \textbf{TruthfulQA \citep{figueras2025truthknowslanguageevaluating} (6-shot, MC2 variant):} This benchmark measures a model's truthfulness and its ability to avoid generating common falsehoods. We use the multiple-choice (MC2) variant with a 6-shot setup, which is a standard configuration for this task.
    
    \item \textbf{WikiText-103 \citep{merity2016pointer} (0-shot):} This dataset is a standard for evaluating a model's fundamental language modeling capability. It is measured in perplexity (ppl), and the standard evaluation protocol is zero-shot, as few-shot prompting is not applicable to perplexity calculation.
\end{itemize}

\subsection{Detailed Standalone AQUA Performance Results}
\label{app:full_standalone_results}

This appendix provides the complete results and a detailed analysis for the standalone AQUA evaluation presented in the main paper. Table \ref{tab:pruning_results_percent_full} shows the performance of both models across the full spectrum of pruning ratios.

\subsubsection*{Detailed Analysis}
For \texttt{Llama-3.1-8B-Instruct}, we can observe that reducing the retention ratio to 0.90 has almost no effect. At $k_{ratio}=0.75$, the performance remains exceptionally strong, with only a 0.02 point increase in perplexity and statistically insignificant changes in accuracy on most tasks. This confirms that a 25\% dimensionality reduction is nearly ``free'' in terms of performance. The first significant drop occurs at $k_{ratio}=0.50$, where the model's mathematical reasoning ability (GSM8K) begins to suffer, although its performance on commonsense tasks like HellaSwag remains robust. This suggests that complex, multi-step reasoning is more sensitive to dimensionality reduction than commonsense inference. The performance degradation accelerates significantly below a ratio of 0.40, with a near-total collapse of reasoning capabilities at 0.20 and below.

The \texttt{OLMoE-1B-7B-Instruct} model shows a similar overall trend but with a more graceful degradation curve. Even at a $k_{ratio}$ of 0.50, the performance drop is minimal across all benchmarks. The degradation becomes more noticeable at 0.40 and 0.30, but it is less severe than what is observed with Llama-3.1 at the same ratios. As discussed in the main text, we attribute this increased resilience to its MHA architecture, where each query has a dedicated key. This allows for greater sparsity in the key vectors compared to the information-dense shared keys in GQA, making them less sensitive to pruning.

\begin{table*}[t!]
\centering
\caption{
    Performance of \texttt{OLMoE-1B-7B-Instruct} \citep{muennighoff2025olmoeopenmixtureofexpertslanguage} using the synergistic \textbf{AQUA-H2O} attention mechanism. The baseline H2O performance ($H2O_{ratio}=1.00$) is denoted by `B'.
}
\label{tab:qma-h2o_olmoe_full}
\resizebox{\textwidth}{!}{%
\begin{tabular}{@{}ccccccccc@{}}
\toprule
\multicolumn{2}{c}{Hyperparameters} & \multicolumn{7}{c}{Benchmark Performance} \\
\cmidrule(r){1-2} \cmidrule(l){3-9}
$H2O_{ratio}$ & $k_{ratio}$ & MMLU & GSM8K & HellaSwag & WinoGrande & \shortstack{TruthfulQA\\MC2} & \shortstack{ARC\\Challenge} & WikiText \\
 & & (acc $\uparrow$) & (acc $\uparrow$) & (acc $\uparrow$) & (acc $\uparrow$) & (acc $\uparrow$) & (acc $\uparrow$) & (ppl $\downarrow$) \\
\midrule
\multirow{4}{*}{0.25} 
 & 0.30 & $0.483 \pm 0.004$ & $0.246 \pm 0.012$ & $0.747 \pm 0.004$ & $0.624 \pm 0.014$ & $0.461 \pm 0.016$ & $0.508 \pm 0.015$ & $12.180$ \\
 & 0.50 & $0.521 \pm 0.004$ & $0.425 \pm 0.014$ & $0.780 \pm 0.004$ & $0.657 \pm 0.013$ & $0.474 \pm 0.016$ & $0.535 \pm 0.015$ & $11.360$ \\
 & 0.75 & $0.527 \pm 0.004$ & $0.431 \pm 0.014$ & $0.786 \pm 0.004$ & $0.675 \pm 0.013$ & $0.477 \pm 0.016$ & $0.544 \pm 0.015$ & $11.360$ \\
 & 1.00 & $0.526 \pm 0.004$ & $0.421 \pm 0.014$ & $0.786 \pm 0.004$ & $0.661 \pm 0.013$ & $0.474 \pm 0.016$ & $0.539 \pm 0.015$ & $11.370$ \\
\midrule
\multirow{4}{*}{0.50} 
 & 0.30 & $0.485 \pm 0.004$ & $0.270 \pm 0.012$ & $0.747 \pm 0.004$ & $0.620 \pm 0.014$ & $0.484 \pm 0.016$ & $0.513 \pm 0.015$ & $12.120$ \\
 & 0.50 & $0.524 \pm 0.004$ & $0.423 \pm 0.014$ & $0.779 \pm 0.004$ & $0.670 \pm 0.013$ & $0.489 \pm 0.016$ & $0.532 \pm 0.015$ & $11.330$ \\
 & 0.75 & $0.529 \pm 0.004$ & $0.445 \pm 0.014$ & $0.784 \pm 0.004$ & $0.665 \pm 0.013$ & $0.487 \pm 0.016$ & $0.535 \pm 0.015$ & $11.330$ \\
 & 1.00 & $0.529 \pm 0.004$ & $0.453 \pm 0.014$ & $0.784 \pm 0.004$ & $0.659 \pm 0.013$ & $0.489 \pm 0.016$ & $0.535 \pm 0.015$ & $11.340$ \\
\midrule
\multirow{2}{*}{0.75} 
 & 0.30 & $0.485 \pm 0.004$ & $0.283 \pm 0.012$ & $0.747 \pm 0.004$ & $0.626 \pm 0.014$ & $0.483 \pm 0.016$ & $0.515 \pm 0.015$ & $12.140$ \\
 & 0.50 & $0.524 \pm 0.004$ & $0.416 \pm 0.014$ & $0.777 \pm 0.004$ & $0.653 \pm 0.013$ & $0.488 \pm 0.016$ & $0.536 \pm 0.015$ & $11.340$ \\
\midrule
\multirow{4}{*}{1.00 (B)} 
 & 0.30 & $0.485 \pm 0.004$ & $0.253 \pm 0.012$ & $0.747 \pm 0.004$ & $0.615 \pm 0.014$ & $0.481 \pm 0.016$ & $0.513 \pm 0.015$ & $12.140$ \\
 & 0.50 & $0.526 \pm 0.004$ & $0.426 \pm 0.014$ & $0.778 \pm 0.004$ & $0.658 \pm 0.013$ & $0.488 \pm 0.016$ & $0.540 \pm 0.015$ & $11.340$ \\
 & 0.75 & $0.529 \pm 0.004$ & $0.453 \pm 0.014$ & $0.783 \pm 0.004$ & $0.669 \pm 0.013$ & $0.485 \pm 0.016$ & $0.542 \pm 0.015$ & $11.330$ \\
 & 1.00 & $0.530 \pm 0.004$ & $0.451 \pm 0.014$ & $0.783 \pm 0.004$ & $0.673 \pm 0.013$ & $0.491 \pm 0.016$ & $0.532 \pm 0.015$ & $11.340$ \\
\bottomrule
\end{tabular}%
}
\end{table*}

\subsection{Detailed AQUA-H2O Results for OLMoE}
\label{app:qma_h2o_olmoe}

This appendix provides the full experimental results for the synergistic AQUA-H2O method applied to the \texttt{OLMoE-1B-7B-Instruct} model. The results, presented in Table \ref{tab:qma-h2o_olmoe_full}, confirm that the performance benefits of combining AQUA with a token eviction strategy generalize to models with standard Multi-Head Attention architectures. As with the Llama-3.1 model, we observe that combining a moderate level of token eviction with a gentle AQUA pruning ratio maintains performance very close to the baseline, demonstrating the versatility of our approach.

\subsection{Supplementary AQUA-Memory Benchmark Results}
\label{app:full_aqua_memory_results}

This appendix contains the supplementary benchmark results for the AQUA-Memory experiment, corresponding to the results presented in Table~\ref{tab:qma-memory-results-main} in the main paper. Table~\ref{tab:qma-memory-results-supplementary} provides the detailed performance metrics for the remaining evaluated tasks.

\begin{table*}[h!]
\centering
\caption{
    Supplementary performance of \texttt{Llama-3.1-8B-Instruct} \citep{grattafiori2024llama3herdmodels} with the \textbf{AQUA-Memory} attention mechanism. Benchmarks are abbreviated: WinoGrande (WG), TruthfulQA MC2 (TQA), and ARC Challenge (ARC). Baseline is in \textbf{bold}.
}
\label{tab:qma-memory-results-supplementary}
\small 
\begin{tabular}{@{}lcccccc@{}}
\toprule
& \multicolumn{3}{c}{Hyperparameters} & \multicolumn{3}{c}{Supplementary Benchmarks} \\
\cmidrule(lr){2-4} \cmidrule(l){5-7}
Attn. Type & $s_{ratio}$ & $k_{ratio}$ & $E_{ratio}$ & WG & TQA & ARC \\
& & & & (acc $\uparrow$) & (acc $\uparrow$) & (acc $\uparrow$) \\ 
\midrule
Full Attn. & --- & --- & 1.000 & $\mathbf{0.755 \pm 0.012}$ & $\mathbf{0.551 \pm 0.016}$ & $\mathbf{0.647 \pm 0.014}$ \\
\midrule
\multirow{6}{*}{\shortstack{AQUA+\\Memory}} 
 & \multirow{3}{*}{0.10} & 0.75 & 0.675 & $0.741 \pm 0.012$ & $0.541 \pm 0.016$ & $0.637 \pm 0.014$ \\
 & & 0.90 & 0.810 & $0.742 \pm 0.012$ & $0.542 \pm 0.016$ & $0.641 \pm 0.014$ \\
 & & 1.00 & 0.900 & $0.747 \pm 0.012$ & $0.542 \pm 0.016$ & $0.642 \pm 0.014$ \\
\cmidrule(l){2-7}
 & \multirow{3}{*}{0.25} & 0.75 & 0.563 & $0.665 \pm 0.013$ & $0.511 \pm 0.016$ & $0.575 \pm 0.014$ \\
 & & 0.90 & 0.675 & $0.669 \pm 0.013$ & $0.511 \pm 0.016$ & $0.574 \pm 0.014$ \\
 & & 1.00 & 0.750 & $0.669 \pm 0.013$ & $0.511 \pm 0.016$ & $0.575 \pm 0.014$ \\
\bottomrule
\end{tabular}
\end{table*}

\subsection*{Qualitative Analysis of Generative Coherence under AQUA Pruning}

To complement our quantitative benchmark results, this section provides a qualitative analysis of how generative coherence is affected by varying levels of AQUA based pruning. We conducted a simple generative task to observe the model's behavior as we decrease the $k_{ratio}$ hyperparameter, which controls the percentage of dimensions retained during the attention computation.

\begin{table*}[h!]
\centering
\caption{Generated summaries from \texttt{Llama-3.1-8B-Instruct} \citep{grattafiori2024llama3herdmodels} with varying levels of AQUA pruning.}
\label{tab:qualitative_analysis}
\begin{tabularx}{\textwidth}{@{} c X @{}}
\toprule
\textbf{AQUA Ratio ($k_{ratio}$)} & \textbf{Generated Summary} \\
\midrule
\textbf{1.0 (Baseline)} & The text is a well-known pangram, a sentence that uses all the letters of the alphabet at least once. It is often used as a demonstration of a font or keyboard's capabilities. The sentence itself describes a scene where a quick, brown fox jumps over a lazy dog. \\
\addlinespace
\textbf{0.90} & The text is a well-known pangram, a sentence that uses all the letters of the alphabet at least once. It is often used as a demonstration of a font or keyboard's capabilities. The sentence itself describes a scene where a quick brown fox jumps over a lazy dog. \\
\addlinespace
\textbf{0.75} & The text is a well-known pangram, a sentence that uses all the letters of the alphabet at least once. It is often used as a demonstration of a font or keyboard's capabilities. The sentence itself describes a scene where a quick brown fox jumps over a lazy dog. \\
\addlinespace
\textbf{0.50} & The given text is a well-known pangram, a sentence that uses all the letters of the alphabet at least once. The text is: ``The quick brown fox jumps over the lazy dog.'' It is often used as a demonstration of a font or keyboard's capabilities. \\
\addlinespace
\textbf{0.40} & The text is a well-known pangram, a sentence that uses all the letters of the alphabet at least once. It is often used as an example in typing and typing games. The sentence is a simple and short phrase describing a scene where a quick brown fox jumps over a lazy dog. \\
\addlinespace
\textbf{0.30} & \textit{The text is a well-known English idiom, often used to demonstrate a sentence that is grammatically correct but semantically nonsensical.} \\
\addlinespace
\textbf{0.20} & \textit{There is no text to summarize.} \\
\bottomrule
\end{tabularx}
\end{table*}

\subsubsection*{Experimental Setup}
We prompted the \texttt{Llama-3.1-8B-Instruct} model \citep{grattafiori2024llama3herdmodels}, modified with our AQUA mechanism, to perform a basic summarization task.
\begin{itemize}[leftmargin=*, noitemsep]
    \item \textbf{Prompt:} ``Give me a summary of the following text: The quick brown fox jumps over the lazy dog.''
    \item \textbf{Decoding Strategy:} We used a deterministic decoding strategy (do\_sample=False) to ensure that any variation in the output is directly attributable to the change in the $k_{ratio}$ and not to sampling randomness.
\end{itemize}

\subsubsection*{Results and Analysis}
The generated responses for different $k_{ratio}$ values are presented in Table \ref{tab:qualitative_analysis}.

The results illustrate a clear and graceful degradation profile, followed by a sharp collapse in coherence:
\begin{itemize}
    \item \textbf{Graceful Degradation (1.0 down to 0.40):} From the baseline down to a $k_{ratio}$ of 0.40, the model correctly identifies the text as a pangram and provides a factually accurate summary. The responses are nearly identical down to a ratio of 0.75. At 0.50 and 0.40, the phrasing changes slightly, but the core semantic content remains perfectly intact. This aligns with our quantitative results, showing that a significant portion of the attention computation can be pruned with minimal impact on the model's knowledge and reasoning abilities.
    \item \textbf{Semantic Failure (at 0.30):} A critical failure occurs at a $k_{ratio}$ of 0.30. The model loses its ability to correctly identify the pangram and instead misclassifies it as a ``semantically nonsensical'' idiom. This represents the point where the information loss from pruning becomes too great, leading to a fundamental error in reasoning.
    \item \textbf{Complete Collapse (at 0.20):} At a $k_{ratio}$ of 0.20, the model's capabilities collapse entirely. It fails to even recognize the presence of the input text, indicating a catastrophic failure in the attention mechanism's ability to process information.
\end{itemize}
This qualitative analysis provides an intuitive demonstration of the trade-offs involved in our method. It confirms that AQUA offers a robust ``sweet spot'' where efficiency can be gained with negligible performance loss, while also clearly defining the operational limits beyond which model coherence is compromised. Also, please note that this property would differ from model to model and architecture to architecture.

\end{document}

%% file: math_commands.tex
\usepackage{amsmath,amsfonts,bm}

\def\eqref#1{equation~\ref{#1}}

\def\1{\bm{1}}

\DeclareMathAlphabet{\mathsfit}{\encodingdefault}{\sfdefault}{m}{sl}
\SetMathAlphabet{\mathsfit}{bold}{\encodingdefault}{\sfdefault}{bx}{n}

